\documentclass[11pt]{article}
\usepackage{graphicx} 
\usepackage{bbm}
\usepackage{geometry}
\usepackage{color}
\usepackage{amsfonts}
\usepackage{algorithm}
\usepackage{algorithmic}
\usepackage{latexsym, amssymb, amsmath, amscd, amsthm, amsxtra}
\usepackage{mathtools}
\usepackage{enumerate}
\usepackage[all]{xy}
\usepackage{mathrsfs}
\usepackage{fancyhdr}
\usepackage{listings}
\usepackage{hyperref}
\usepackage{enumitem}

\def\11{\mathbbm{1}}

\def\ER{Erd\H{o}s-R\'enyi\ }

\def\Fop{\operatorname{F}}

\newcommand{\op}{\operatorname{op}}

\newtheorem{thm}{Theorem}[section]

\newtheorem{proposition}[thm]{Proposition}

\newtheorem{lemma}[thm]{Lemma}

\newtheorem{defn}[thm]{Definition}

\newtheorem{claim}[thm]{Claim}

\newtheorem{remark}[thm]{Remark}

\numberwithin{equation}{section}

\makeatletter
\newenvironment{breakablealgorithm}
{
		\begin{center}
			\refstepcounter{algorithm}
			\hrule height.8pt depth0pt \kern2pt
			\renewcommand{\caption}[2][\relax]{
				{\raggedright\textbf{\ALG@name~\thealgorithm} ##2\par}%
				\ifx\relax##1\relax 
				\addcontentsline{loa}{algorithm}{\protect\numberline{\thealgorithm}##2}%
				\else 
				\addcontentsline{loa}{algorithm}{\protect\numberline{\thealgorithm}##1}%
				\fi
				\kern2pt\hrule\kern2pt
			}
		}{
		\kern2pt\hrule\relax
	\end{center}
}
\makeatother

\hypersetup{colorlinks=true,linkcolor=blue}
\title{Robust Random Graph Matching in Dense Graphs via an Approximate Message Passing Type Algorithm}
\author{Zhangsong Li\\Peking University}
\date{}
\begin{document}
\maketitle

\begin{abstract}
    In this paper, we focus on the matching recovery problem between a pair of correlated Gaussian Wigner matrices with a latent vertex correspondence. We are particularly interested in a robust version of this problem such that our observation is a perturbed input $(A+E,B+F)$ where $(A,B)$ is a pair of correlated Gaussian Wigner matrices and $E,F$ are adversarially chosen matrices supported on an unknown $\epsilon n * \epsilon n$ principal minor of $A,B$, respectively. We propose an approximate message passing (AMP) type iterative algorithm that succeeds in polynomial time as long as the correlation $\rho$ between $(A,B)$ is a non-vanishing constant and $\epsilon = o\big( \tfrac{1}{(\log n)^{20}} \big)$. A key distinction from standard AMP is the introduction of a time-dependent matrix multiplication step within the iteration, which simultaneously enlarges the feature dimension and cancels the correlation during the iteration. 

    The main methodological inputs for our result are the iterative random graph matching algorithm proposed in \cite{DL22+, DL23+} and the spectral preprocessing procedure proposed in \cite{IS24+}. To the best of our knowledge, our algorithm is the first efficient random graph matching type algorithm that is robust under any adversarial perturbations of $n^{1-o(1)}$ size.\footnote{This paper was presented in part at The 38th Annual Conference on Learning Theory (COLT 2025). This version includes complete mathematical proofs and detailed discussions.}
\end{abstract}



\section{Introduction}

In this paper, we study the problem of matching two correlated random matrices, and we consider the case of symmetric matrices in order to be consistent with the graph matching problem. More precisely, we will let these two matrices be the adjacency matrices of a pair of correlated weighted random graphs, which is defined as follows. Let $\operatorname{U}_n$ be the set of unordered pairs $(i,j)$ with $1 \leq i \neq j \leq n$. 

\begin{defn}[Correlated weighted random graphs]{\label{def-correlated-graphs}}
    Let $\pi_*$ be a latent permutation on $[n]=\{ 1,\ldots,n \}$. We generate two weighted random graphs on the common vertex set $[n]$ with adjacency matrices $A$ and $B$ such that given $\pi_*$, we have $(A_{i,j}, B_{\pi_*(i),\pi_*(j)}) \sim \mathbf F$ independent among all $(i,j) \in \operatorname{U}_n$ where $\mathbf F$ is the law of a pair of correlated random variables. Of particular interest are the following special cases:
    \begin{itemize}
        \item \emph{Correlated Gaussian Wigner model.} In this case, we let $\mathbf F$ be the law of two mean-zero Gaussian random variables with variance $1$ and correlation $\rho$.
        \item \emph{Correlated \ER graph model.} In this case, we let $\mathbf F$ be the law of two Bernoulli random variables with mean $q \leq \frac{1}{2}$ and correlation $\rho$.
    \end{itemize}
\end{defn}

Given two correlated weighted random graphs $(A,B)$, our goal is to recover the latent vertex correspondence $\pi_*$. To achieve such a task, one can consider the following natural optimization problem
\begin{equation}{\label{eq-optimitation}}
    \max_{\Pi} \Big\{ \operatorname{tr}(A\Pi B\Pi^{\top}) \Big\} \quad \mbox{s.t.} \quad \Pi \in\mathfrak S_n
\end{equation}
where $\mathfrak{S}_n$ is the set of $n*n$ permutation matrices. However, solving \eqref{eq-optimitation} is difficult as there are exponentially many choices for $\Pi\in\mathfrak S_n$. For both the correlated Gaussian Wigner model and the correlated \ER graph model, by the collective effort of the community, it is fair to say that our understanding of the statistical and computational aspects of the matching recovery problem in both models are more or less satisfactory (see Section~\ref{subsec:related-works} for a more detailed overview). However, there is a new fascinating issue that arises in the context of the works on matching recovery, namely the \emph{robustness issue}: many of the efficient algorithms used to achieve matching recovery are believed to be fragile in the sense that adversarially modifying a small fraction of edges could fool the algorithm into outputting a result which deviates strongly from the true underlying matching $\pi_*$. The reason is that these algorithms are either based on enumeration of sophisticated subgraph structures (see, e.g., \cite{BCL+19, MWXY23, GMS24} for example) or are based on delicate spectral properties of the adjacency matrices (see, e.g., \cite{FMWX23a, FMWX23b} where the authors design an efficient algorithm based on all the eigenvectors of the adjacency matrix) that can be affected disproportionately by adding small cliques or other ``undesired'' subgraph structure. Thus, a natural question is whether we can find efficient random graph matching algorithms that are robust under a small fraction of \emph{adversarial} perturbations. To be more precise, we will consider the problem of matching the following \emph{corrupted} correlated weighted random graph model.

\begin{defn}[Corrupted correlated weighted random graphs]{\label{def-corrupted-correlated-graphs}}
    We define two weighted random graphs, represented by their adjacency matrices $(A',B')$, as a pair of $\epsilon$-corrupted correlated weighted random graphs if there exists a pair of correlated weighted random graphs $(A,B)$ with correlation $\rho$ such that $(A',B')=(A+E,B+F)$. Here $E,F$ are arbitrary symmetric matrices supported on an (unknown) $\epsilon n * \epsilon n$ principal minor of $A,B$, respectively (note that we allow $E$ and $F$ to depend on $A$ and $B$, and the supports of $E$ and $F$ are not necessarily the same).  
\end{defn}

In this paper we will focus on corrupted correlated Gaussian Wigner model, in which the observations are two $n*n$ matrices $(A',B')$ such that there exists a pair of correlated Gaussian Wigner matrices $(A,B)$ with correlation $\rho$ satisfying $(A',B')=(A+E,B+F)$. Our main result can be summarized as follows:

\begin{thm}{\label{MAIN-THM}}
    Suppose $\rho\in(0,1)$ is a constant and $\epsilon=o\big( \tfrac{1}{(\log n)^{20}} \big)$. Then for a pair of $\epsilon$-corrupted Gaussian Wigner matrices with correlation $\rho$ (we denote them by $A',B'$), there exists a constant $C=C(\rho)$ and an algorithm (See Algorithm~\ref{algo:robust-matching}) with $O(n^{C})$ running time that takes $(A',B')$ as input and outputs the latent matching $\pi_*$ with probability tending to $1$ as $n\to \infty$.
\end{thm}


\subsection{Related works}{\label{subsec:related-works}}

{\em Random graph matching.} Graph matching (also known as network alignment) refers to the problem of finding the bijection between the vertex sets of two graphs that maximizes the total number of common edges. When the two graphs are exactly isomorphic to each other, this reduces to the classical graph isomorphism problem, for which the best known algorithm runs in quasi-polynomial time \cite{Babai16}. In general, graph matching is an instance of the \emph{quadratic assignment problem} \cite{BCP+98}, which is known to be NP-hard to solve or even approximate \cite{MMS10}. Motivated by real-world applications (such as social network deanonymization \cite{NS08, NS09}, computer vision \cite{BBM05}, natural language processing \cite{HNM05} and computational biology \cite{SXB08}) as well as the need to understand the average-case computational complexity, a recent line of work is devoted to the study of statistical theory and efficient algorithms for graph matching under statistical models, by assuming the two graphs are randomly generated with correlated edges under a \emph{hidden} vertex correspondence.  

From the theoretical point of view, the correlated Gaussian Wigner model and the correlated \ER graph model are arguably the most canonical and the most widely studied models for studying graph matching in the \emph{average setting}. Recent efforts have yielded information-theoretic thresholds for both exact and partial matching recovery \cite{CK16, CK17, CKMP19, HM23, WXY22, WXY23, GML21, DD23a, DD23b, Du25+} and a variety of efficient graph matching algorithms with performance guarantees have been developed \cite{YG13, BSH19, BCL+19, DMWX21, FMWX23a, FMWX23b, GM20, GML24, MRT21, MRT23, GMS24, MWXY21+, MWXY23, DL22+, DL23+}. We now focus on the algorithmic aspect of this problem since it is more relevant to our work. The state-of-the-art algorithm can be summarized as follows: in the sparse regime, efficient matching algorithms are available when the correlation exceeds the square root of Otter’s constant (the Otter’s constant is approximately 0.338) \cite{MWXY21+, MWXY23, GML24, GMS24}; in the dense regime, efficient matching algorithms exist as long as the correlation exceeds an arbitrarily small constant \cite{DL22+, DL23+}. Roughly speaking, the separation between the sparse and dense regimes mentioned above depends on whether the average degree of the graph grows polynomially or sub-polynomially. In addition, while proving the hardness of typical instances of the graph matching problem remains challenging even under the assumption of P$\neq$NP, evidence based on the analysis of a specific class known as low-degree polynomials from \cite{DDL23+, Li25} indicates that the state-of-the-art algorithms may essentially capture the correct computational thresholds.

{\em Robust algorithms.} The problem of finding robust algorithms for solving statistical estimation and random optimization problems has garnered significant attention in recent years. A prominent example in this scope is the problem of robust community recovery in sparse stochastic block models. In recent years, a large body of work has focused on the problem of designing community recovery algorithms where an adversary may arbitrarily modify $\Omega(n)$ edges (see, e.g., \cite{MS16, DdNS22, MRW24}). Other important robust algorithms include linear regression \cite{BP21}, mean and moment estimation \cite{KSS18}, clustering mixture model \cite{HL18}, and so on. 

In the context of random graph matching, previous robustness results mainly focus on the information-theoretic side. For instance, in \cite{AH23+} the authors considered the behavior of the maximum overlap estimator and the $k$-core estimator for matching recovery in a pair of correlated \ER graphs with corruption (although their definition of corruption is a bit different from ours). They also conduct valuable numerical experiments which imply that several widely used graph matching algorithms (e.g., the spectral graph matching algorithm in \cite{FMWX23a, FMWX23b} and the degree profile matching algorithm in \cite{DMWX21}) behave poorly even when only a small portion of the graph is corrupted. In fact, it seems that simply planting an arbitrary $\Theta(\sqrt{n})$ size clique in both graphs will significantly change the spectral properties and the degree distribution of the graph, causing these algorithms to fail. This raises the important question of finding computationally feasible algorithms that are robust in the presence of adversarial corruption. We answer this problem partly by proposing an efficient random graph matching algorithm which is robust under any ${\frac{n}{(\log n)^{20}} * \frac{n}{(\log n)^{20}}}$ adversarial perturbations, thus improving the robustness guarantees by a factor of $\mathrm{poly}(n)$.

{\em Approximate message passing.} Approximate Message Passing (AMP) is a family of algorithmic methods which generalizes matrix power iteration. Originated from statistical physics and graphical models \cite{TAP77, KF09, Montanari12, Bolthausen14}, it has emerged as a popular class of first-order iterative algorithms that find diverse applications in both statistical estimation problems and probabilistic analyses of statistical physics models. Some notable examples include compressed sensing \cite{DMM09}, sparse Principal Components Analysis (PCA) \cite{DM14}, linear regression \cite{DMM09, BM11, KMS+12}, non-negative PCA \cite{MR15}, perceptron models \cite{DS19, FW24, BNSX22, FLS22} and more (a more extensive list can be found in the survey \cite{FVRS22}).

One major limitation of the original AMP algorithms is that they are not robust under small adversarial perturbations. To address this issue, in \cite{IS24, IS24+} the authors propose to apply AMP algorithms using ``suitably preprocessed'' initialization and data matrix. Building on this idea, they found the first robust AMP-based iterative algorithm for the non-negative PCA problem.

\subsection{Notations}

We record in this subsection some notation conventions. Recall that the observation $(A',B')$ are two $n*n$ matrices with $(A',B')=(A+E,B+F)$. Denote $Q,R$ to be the support of $E,F$, respectively. We then have
\begin{align*}
    E_{i,j}=0 \mbox{ for all } (i,j) \not \in Q \times Q \mbox{ and } F_{i,j}=0 \mbox{ for all } (i,j) \not \in R \times R \,.
\end{align*}
Note that $A,B,E,F,Q,R$ are inaccessible to the algorithm. Given two random variables $X,Y$ and a $\sigma$-algebra $\mathfrak{F}$, the notation $X|{\mathfrak{F}} \overset{d}{=} Y|{\mathfrak{F}}$ means that for any integrable function $\phi$ and for any bounded random variable $Z$ measurable on $\mathfrak{F}$, we have $\mathbb{E}[\phi(X)Z] = \mathbb{E}[\phi(Y)Z]$. In words, $X$ is equal in distribution to $Y$ conditioned on $\mathfrak{F}$. When $\mathfrak{F}$ is the trivial $\sigma$-field, we simply write $X \overset{d}{=} Y$. 

We also need some standard notations in linear algebra. For a matrix or a vector $M$, we will use $M^{\top}$ to denote its transpose. For an $m*m$ matrix $M=(a_{ij})_{m*m}$, if $M$ is symmetric we let $\varsigma_1(M) \geq \varsigma_2(M) \geq \ldots \geq \varsigma_m(M)$ be the eigenvalues of $M$. Denote by $\mathrm{rank}(M)$ the rank of the matrix $M$. For two $l*m$ matrices $M_1$ and $M_2$, we define their inner product to be
\begin{align*}
    \langle M_1,M_2 \rangle:=\sum_{i=1}^l \sum_{j=1}^m M_1(i,j)M_2(i,j) \,.
\end{align*}
We also define the Frobenius norm, operator norm, and $\infty$-norm of $M$ respectively by
\begin{align*}
    \| M \|_{\operatorname{F}} = \langle M,M \rangle^{\frac{1}{2}}, \
    \| M \|_{\operatorname{op}} = \varsigma_1(M M^{\top})^{\frac{1}{2}}, \ 
    \| M \|_{\infty} = \max_{ \substack{ 1 \leq i \leq l \\ 1 \leq j \leq m } } |M_{i,j}| 
\end{align*}
where $\mathrm{tr}(\cdot)$ is the trace of a square matrix. Denote $\mathfrak S_n$ to be the set of all permutations on $[n]$. For a bijection $\sigma:U \to V$ and a matrix $M$ with rows and columns indexed by $V,W$ respectively, we define $M(\sigma)$ to be the matrix indexed by $U,W$, with entries given by $M(\sigma)_{i,j} = M_{\sigma(i),j}$. For any $d*l$ matrix $M$ and two index sets $I \subset [d],J\subset [l]$, we denote $M_{I\times J}$ to be the matrix indexed by $I \times J$ with $(M_{I\times J})_{i,j}=M_{i,j}$ for $i\in I,j \in J$. We will use $\mathbb{I}_{d*d}$ to denote the $d*d$ identity matrix (and we drop the subscript if the dimension is clear from the context). Similarly, we denote $\mathbb{O}_{m*d}$ the $m*d$ zero matrix and denote $\mathbb{J}_{m*d}$ the $m*d$ matrix with all entries being 1. The indicator function of a set $A$ is denoted by $\mathbf{1}_{A}$.  

For any two positive sequences $\{a_n\}$ and $\{b_n\}$, we write equivalently $a_n=O(b_n)$, $b_n=\Omega(a_n)$, $a_n\lesssim b_n$ and $b_n\gtrsim a_n$ if there exists a positive absolute constant $c$ such that $a_n/b_n\leq c$ holds for all $n$. We write $a_n=o(b_n)$, $b_n=\omega(a_n)$, $a_n\ll b_n$, and $b_n\gg a_n$ if $a_n/b_n\to 0$ as $n\to\infty$. We write $a_n =\Theta(b_n)$ if both $a_n=O(b_n)$ and $a_n=\Omega(b_n)$ hold.

\section{Algorithms and discussions}{\label{sec:alg-and-discussions}}

\subsection{Overview of our algorithm}

We begin with a high-level overview of our algorithmic approach. The full, formal algorithm is presented in Section~\ref{sec:formal-algorithm}. 
\begin{breakablealgorithm}{\label{algo:robust-matching-informal}}
\caption{Robust Gaussian Matrix Matching Algorithm (Informal)}
    \begin{algorithmic}[1]
    \STATE {\bf Input}: corrupted correlated Gaussian Wigner matrices $A',B'$, correlation $\rho$, corruption fraction $\epsilon$. 
    \STATE Use the preprocessing step in Section~\ref{subsec:preprocessing} to obtain two matrices $\widehat{\mathscr A},\widehat{\mathscr B}$.
    \STATE {Calculate $\varepsilon_0$ according to \eqref{eq-def-varepsilon-0} and choose $K_0$ satisfying \eqref{eq-def-K-0}.}
    \STATE Generate matrices $\Phi^{(t)},\Psi^{(t)},\Xi^{(t)}$ as in Section~\ref{subsec:spec-subroutine}.
    \STATE List all sequences with $K_0$ distinct elements in $[n]$ by $\mathsf{V}_1, \mathsf{V}_2, \ldots, \mathsf{V}_{\mathtt{M}}$.
    \FOR{$\mathtt{i,j}=1,\ldots,\mathtt{M}$}
    \STATE Calculate initializations $\widehat{f}^{(0)},\widehat{g}^{(0)} \in \mathbb R^{n*K_0}$ as in Section~\ref{subsec:initialization}. 
    \STATE Iteratively calculate $(\widehat{f}^{(t)},\widehat{g}^{(t)}, \widehat{h}^{(t)},\widehat{\ell}^{(t)}:t \geq 0)$ as in Section~\ref{subsec:vector-AMP}, where $\widehat{f}^{(t)}, \widehat{g}^{(t)} \in \mathbb R^{n*K_t}$ and $\widehat{h}^{(t)}, \widehat{\ell}^{(t)} \in \mathbb R^{n*\frac{K_t}{12}}$;
    \STATE Stop at some $t=t^{*}$ defined in \eqref{eq-def-t^*};
    \STATE Solve the linear assignment problem between $\widehat{h}^{(t^*)}$ and $\widehat{\ell}^{(t^*)}$ (see \eqref{eq-linear-assignment}); the solution is denoted as $\pi_{\mathtt i,\mathtt j}$.
    \STATE Run the {rounding} algorithm (Algorithm~\ref{algo:seeded-matching}) with input $\pi_{\mathtt i,\mathtt j}$ and obtain $\widehat{\pi}_{\mathtt i,\mathtt j}$.
    \ENDFOR
    \STATE Find ${\widehat{\pi}}_{\mathtt{i_*,j_*}}$ which maximizes the alignment between $A'$ and $B'$ (i.e., maximizes \eqref{equ-def-final-pi-hat}) among all $\{ {\widehat{\pi}}_{\mathtt i,\mathtt j} \}$.
    \STATE {\bf Output:} ${\widehat{\pi}}_{\mathtt{i}_*,\mathtt{j}_*}$.
    \end{algorithmic}
\end{breakablealgorithm}
Our algorithm can be informally summarized in the following steps:

{\bf Step~1: Preprocessing.} We introduce a preprocessing step that transforms the observed symmetric matrices $(A',B')$ into processed matrices $(\widehat{\mathscr A}, \widehat{\mathscr B})$ {that are not symmetric}. This transformation serves two key purposes: first, it effectively decouples the slight statistical dependencies induced by the symmetry of $A'$ and $B'$; second, it allows us to zero out a small fraction of rows and columns while rigorously controlling the operator norm of the resulting matrices $(\widehat{\mathscr A},\widehat{\mathscr B})$. This operator norm bound is essential, as it directly governs the effective strength of the adversarial corruption.

{\bf Step~2: Initialization.} The core algorithmic idea is to iteratively construct vertex features $\{ \widehat{f}^{(t)}_i, \widehat{g}^{(t)}_i \in \mathbb R^{K_t}: i \in [n], t \geq 0 \}$ such that the inner product $\langle \widehat{f}^{(t)}_i, \widehat{g}^{(t)}_j \rangle$ is large precisely when $j=\pi_*(i)$ and small otherwise. If a set of $K_0$ correct seed pairs $\{ (i,\pi_*(i)):i \in \mathsf V \}$ were available, one could initialize these features using the correlation between $(\widehat{\mathscr A}_{i,j}:i \in \mathsf V)$ and $(\widehat{\mathscr B}_{\pi_*(i),\pi_*(j)}:i \in \mathsf V)$. In order to address the fact that we do not have seeds, we essentially just take arbitrary $K_0$ vertices from $[n]$ and try all possible pairings $(V,\mathsf V)$ to these vertices; this incurs a multiplicative cost of ${n^{2K_0}}$, which remains polynomial in $n$.

{\bf Step~3: Spectral subroutine.} We will introduce a spectral subroutine which enables us to efficiently construct matrices $\Phi^{(t)},\Psi^{(t)},\Xi^{(t)}$ and prove certain spectral properties regarding them. These matrices play a crucial role in our iterative construction of the features $\{ \widehat{f}^{(t)}_i, \widehat{g}^{(t)}_i: i \in [n] \}$, as they simultaneously enlarge the feature dimension and cancel {the correlation during the iteration} (see Step~4 below for further explanation).

{\bf Step~4: Iteration.} The core of our algorithm is an iterative update of the feature vectors $\{ \widehat{f}^{(t)}_i, \widehat{g}^{(t)}_i, \widehat{h}^{(t)}_i, \widehat{\ell}^{(t)}_i: i \in [n] \}$ starting from the initialization $\{ \widehat{f}^{(0)}_i, \widehat{g}^{(0)}_i: i \in [n] \}$. This iteration is a variant of approximate message passing (AMP) applied separately to $\widehat{\mathscr A}$ and $\widehat{\mathscr B}$ respectively, and it differs from standard AMP in two key ways. First, although AMP can handle vectorial features \cite{RSF19, LM24+, CLM25+}, our iteration incorporates a time-dependent matrix multiplication that intentionally increases the feature dimension. This increase compensates for the decay in per-coordinate correlation between $\widehat{h}^{(t)}_i$ and $\widehat{\ell}^{(t)}_{\pi_*(i)}$ over time (see Remark~\ref{remark-increase-dimension}). Second, instead of using an Onsager term to cancel the iterative correlation, we cancel the correlation directly via careful choice of the multiplicative matrix (see Remark~\ref{remark-no-Onsager-term}). This eliminates the need for the Onsager correction, substantially simplifying the analysis. We note that similar techniques to avoid introducing Onsager correction terms have also been used in the field of orthogonal AMP \cite{ML17, CLM25+}.

{\bf Step~5: Finishing and rounding.} Once our iteration progresses to suitable time $t=t^*$, enough ``total signal'' has been accumulated. At this stage, the quadratic assignment problem can be essentially reduced to a linear assignment problem: we find the permutation $\widehat{\pi}$ that maximizes the alignment between $\sum_i \langle \widehat{h}^{(t^*)}_{\widehat{\pi}(i)}, \widehat{\ell}^{(t^*)}_i \rangle$. In this way, we have essentially reduced the quadratic assignment problem into an easier linear assignment problem, which can be solved efficiently by the Hungarian algorithm. We prove that the resulting permutation matches a $(1-o(1))$ fraction of the coordinates of the true permutation $\pi_*$. Finally, we apply a rounding scheme to refine the almost-correct matching $\widehat{\pi}$ into an exact recovery of the underlying permutation. {As mentioned in Step~11 of Algorithm~\ref{algo:robust-matching-informal}, for each candidate seeded pairing $(V,\mathsf V)$, we apply the rounding procedure to obtain an estimator $\widehat{\pi}_{(V,\mathsf V)}$. When $(V,\mathsf V)$ coincides with the correct seeded pairs, the resulting estimator satisfies $\widehat{\pi}_{(V,\mathsf V)}=\pi_*$. Among all such candidates, we then select the final estimator by choosing the $\widehat{\pi}_{(V,\mathsf V)}$ that maximizes the alignment score between $A$ and $B$.}

In the rest of the section, we describe in detail our algorithm, which consists of a few steps including preprocessing (see Section~\ref{subsec:preprocessing}), initialization (see Section~\ref{subsec:initialization}), spectral subroutine (see Section~\ref{subsec:spec-subroutine}), iteration (see Section~\ref{subsec:vector-AMP}), finishing and rounding (see Section~\ref{subsec:rounding}). We formally present our algorithm and analyze the time complexity of the algorithm in Section~\ref{sec:formal-algorithm} (see Algorithm~\ref{algo:robust-matching} and Proposition~\ref{prop-time-complexity}).

\subsection{Preprocessing}{\label{subsec:preprocessing}}

The first step of our algorithm is to make some preprocessing on $A',B'$ for technical convenience. We first make a technical assumption that we only need to consider the case when $\rho$ is a sufficiently small constant, which can be easily achieved by deliberately adding i.i.d.\ noise to each $\{ A'_{i,j} \}$ and $\{ B'_{i,j} \}$. Sample i.i.d.\ $\mathcal N(0,1)$ random variables $\{ G_{i,j}, H_{i,j}:1\leq i<j\leq n \}$ and let
\begin{equation}{\label{eq-def-widehat-mathscr-A,B-Gaussian}}
\begin{aligned}
    \widehat{A}'_{i,j} = \frac{ A'_{i,j} + G_{i,j} }{ \sqrt{2} }, \ \widehat{B}'_{i,j} = \frac{ B'_{i,j} + H_{i,j} }{ \sqrt{2} } \mbox{ for } i<j \,, \\
    \widehat{A}'_{i,j} = \frac{ A'_{i,j} - G_{j,i} }{ \sqrt{2} }, \ \widehat{B}'_{i,j} = \frac{ B'_{i,j} - H_{j,i} }{ \sqrt{2} } \mbox{ for } i>j \,.
\end{aligned}
\end{equation}
Note that $\widehat{A}'$ and $\widehat{B}'$ are no longer symmetric matrices, at variance with $A'$ and $B'$.
Now we introduce the spectral cleaning procedure. Informally speaking, this procedure enables us to zero-out $4\epsilon n$ rows and columns of $\widehat{A}', \widehat{B}'$ respectively to get two ``cleaned'' matrices $\widehat{\mathscr A},\widehat{\mathscr B}$ with $\| \widehat{\mathscr A} \|_{\operatorname{op}}, \| \widehat{\mathscr B} \|_{\op}\leq 10\sqrt{n}$ with probability $1-o(1)$. Note that $(\widehat A',\widehat B')=(\widehat A, \widehat B) + (\widehat E,\widehat F)$, where 
\begin{align}
    & \widehat{A}_{i,j} = \frac{ A_{i,j} + G_{i,j} }{ \sqrt{2} }, \quad \widehat{B}_{i,j} = \frac{ B_{i,j} + H_{i,j} }{ \sqrt{2} } \mbox{ for } i<j \,, \label{eq-de-f-widehat-A'} \\
    & \widehat{A}_{i,j} = \frac{ A_{i,j} - G_{j,i} }{ \sqrt{2} }, \quad \widehat{B}_{i,j} = \frac{ B_{i,j} - H_{j,i} }{ \sqrt{2} } \mbox{ for } i>j \,. \label{eq-de-f-widehat-B'} \\
    & \widehat{E}_{i,j} = \frac{1}{\sqrt{2}} E_{i,j}, \quad \widehat{F}_{i,j} = \frac{1}{\sqrt{2}} F_{i,j} \mbox{ for } i<j \,, \label{eq-de-f-widehat-E} \\
    & \widehat{E}_{i,j} = \frac{1}{\sqrt{2}} E_{i,j}, \quad \widehat{F}_{i,j} = \frac{1}{\sqrt{2}} F_{i,j} \mbox{ for } i>j \,. \label{eq-de-f-widehat-F}
\end{align}
It is straightforward to verify that $\big\{ \widehat{A}_{i,j} \big\}$ and $\big\{ \widehat B_{i,j} \big\}$ are two families of i.i.d.\ standard normal random variables. Also, we have
\begin{align*}
    \mathrm{Cov}( \widehat{A}_{i,j}, \widehat{B}_{\pi_*(i),\pi_*(j)} ) = \mathrm{Cov}( \widehat{A}_{i,j}, \widehat{B}_{\pi_*(j),\pi_*(i)} ) = \tfrac{\rho}{2} \,.
\end{align*}
We further employ a ``spectral cleaning'' procedure to $\widehat A',\widehat B'$ respectively. Note by \eqref{eq-de-f-widehat-E}, \eqref{eq-de-f-widehat-F} that $\widehat E,\widehat F$ are still supported on $Q,R \subset [n]$ with $|Q|,|R| \leq \epsilon n$ respectively. In addition, since $\widehat A, \widehat B$ are random matrices with i.i.d.\ sub-Gaussian entries, from \cite[Theorem~4.4.5]{Vershynin18} we see that with probability $1-o(1)$ we have $\| \widehat A \|_{\op},\| \widehat B \|_{\op} \leq (2+o(1)) \sqrt{n}$. Our spectral cleaning procedure is a modified version of \cite[Algorithm~3.7]{IS24+}:

\begin{breakablealgorithm}{\label{alg:spectral-cleaning}}
\caption{Spectral Cleaning Algorithm}
    \begin{algorithmic}[1]
    \STATE \textbf{Input}: $n*n$ Matrix $M'$.
    \STATE Let $\mathscr{M}=M'$.
    \WHILE{$\| \mathscr{M} \|_{\op} \geq 10\sqrt{n}$}
    \STATE Compute the unit left singular eigenvector $v=(v_1,\ldots,v_n)$ and unit right singular eigenvector $u=(u_1,\ldots,u_n)$ of $\mathscr{M}$ corresponding to the leading singular value. 
    \STATE Sample $i \in [n]$ with probability $\frac{v_i^2+u_i^2}{2}$. 
    \STATE Zero-out the $i$'th row and column of $\mathscr{M}$.
    \ENDWHILE
    \STATE \textbf{Output}: $\mathscr{M}$. 
    \end{algorithmic}
\end{breakablealgorithm}

Clearly, by running Algorithm~\ref{alg:spectral-cleaning} with input $\widehat{A}',\widehat{B}'$ respectively we get two matrices $\widehat{\mathscr A},\widehat{\mathscr B}$ with $\| \widehat{\mathscr A} \|_{\op}, \| \widehat{\mathscr B} \|_{\op} \leq 10\sqrt{n}$. In addition, denote $S,T \subset [n]$ to be the set of indices of $\widehat{A}',\widehat{B}'$ which are zeroed-out by the algorithm, the following lemma (similar to \cite[Lemma~3.5]{IS24+}) controls the cardinality of $S$ and $T$.
\begin{lemma}{\label{lem-bound-card-T,T'}}
    If the input matrix $M'=M+E$ with $\| M \|_{\op} \leq (2+o(1))\sqrt{n}$ and the support of $E$ (denoted as $Q$) is bounded by $\epsilon n$, then with probability $1-o(1)$ we have Algorithm~\ref{alg:spectral-cleaning} terminates in $4\epsilon n$ steps. In particular, with probability $1-o(1)$ we have $|S|,|T| \leq 4\epsilon n$.
\end{lemma}

The proof of Lemma~\ref{lem-bound-card-T,T'} is postponed to Section~\ref{sec:statement-spectral-cleaning-alg} of the appendix. From now on we will work on $\widehat{\mathscr A}$ and $\widehat{\mathscr B}$.

\subsection{Initialization}{\label{subsec:initialization}}

Before presenting our initialization procedure, we first choose a suitable smooth function $\varphi$ which will be used as the ``denoiser function'' throughout our algorithm. 

\begin{defn}{\label{def-denoiser-function}}
    We choose a smooth function $\varphi(x)$ such that the following conditions hold:
    \begin{enumerate}
        \item[(1)] $\big| \varphi(x) \big|, \big| \varphi'(x) \big|, \big| \varphi''(x) \big| \leq 100$ for all $x \in \mathbb R$ (here $100$ is somewhat arbitrarily chosen). Also $\big| \varphi^{(k)}(x) \big| \leq (100+|x|)^k$ for all $x \in \mathbb R$ and $k \in \mathbb N$.
        \item[(2)] for a standard normal variable $X$, we have $\mathbb E[\varphi(X)]=\mathbb E[\varphi'(X)]=0$ and $\mathbb E[\varphi(X)^2]=1$. In addition, we have $\mathbb E[X^2\varphi(X)],\mathbb E[\varphi''(X)]\neq 0$.
    \end{enumerate}
    In addition, for a pair of standard bivariate normal variables $(X,Y)$ with correlation $u$, we define $\phi : [-1,1] \to \mathbb R$ by
    \begin{equation}{\label{eq-def-phi}}
        \phi(u):= \mathbb E\big[ \varphi(X)\varphi(Y) \big]  \,.
    \end{equation} 
\end{defn}
\begin{remark}
    To see that our choice of $\varphi(x)$ in Definition~\ref{def-denoiser-function} is not empty, we note that we can simply choose
    \begin{align*}
        \varphi(x) = \frac{ \cos(x)-\mathbb E_{X \sim \mathcal N(0,1)}[\cos(X)] }{ \sqrt{\operatorname{Var}_{X \sim\mathcal N(0,1)}[\cos(X)]} } \,.
    \end{align*}
    In this case, using $\mathbb E_{X,Z\sim\mathcal N(0,1)}[\cos(tX+sZ)]=e^{-\frac{1}{2}(t^2+s^2)}$ we can explicitly calculate that
    \begin{align*}
        \phi({u}) = \frac{ e^{-1+u}+e^{-1-u}-2e^{-1} }{ 1+e^{-2}-2e^{-1} } \,.
    \end{align*}
    However, for the simplicity of notations and for the emphasis on the generality of our choice of $\varphi(x)$, throughout this paper we will work on a general $\varphi(x)$ satisfying Items~(1) and (2).
\end{remark}

We need the following properties of $\phi(u)$.
\begin{lemma}{\label{lem-control-Taylor-expansion}}
    We have the following results:
    \begin{enumerate}
        \item[(1)] If we write $\phi(u)=\sum_{m=0}^{\infty} c_m u^m$, then we have $c_0=c_{1}=0,c_2>0$ and there exists a constant $\Lambda=\Lambda(\varphi)$ such that $|c_{k}| \leq \Lambda \cdot 2^k$ for all $k\geq 2$. 
        \item[(2)] We have ${\frac{1}{4}}\phi''(0) \cdot u^2 \leq \phi(u) \leq 2\phi''(0) \cdot u^2$ for all sufficiently small $u$.
    \end{enumerate}    
\end{lemma}
\begin{proof}
    Note that for bivariate standard normal variables $X,Y$ with correlation $u$, we can write $Y=uX+\sqrt{1-u^2} Z$ where $Z$ is independent with $X$. Thus 
    \begin{align*}
        \phi(u) = \mathbb E\Bigg[ \varphi(X) \varphi\Big( uX+\sqrt{1-u^2} Z \Big) \Bigg] \,.
    \end{align*}
    Thus, direct calculation yields that (note that we have $\mathbb E[X f(X)]=\mathbb E[f'(X)]$ from Gaussian integration by parts)
    \begin{align*}
        c_0 &= \phi(0) = \mathbb E\Big[ \varphi(X) \varphi(Z) \Big] \overset{\text{Item~(2)}}{=} 0 \,; \\
        c_1 &= \phi'(0) = \mathbb E\Big[ X \varphi(X) \varphi'(Z) \Big] \overset{\text{Item~(2)}}{=} 0 \,; \\
        2c_2 &= \phi''(0) = \mathbb E\Big[ X^2 \varphi(X) \varphi''(Z) - \varphi(X) \varphi'(Z) Z \Big] \overset{\text{Item~(2)}}{=} \mathbb E\Big[ X^2 \varphi(X) \Big] \mathbb E\Big[ \varphi''(X) \Big] \\
        &= \mathbb E\Big[ X \varphi'(X) + \varphi(X) \Big] \mathbb E\Big[ \varphi''(X) \Big] = \mathbb E\Big[ X \varphi'(X) \Big] \mathbb E\Big[ \varphi''(X) \Big] = \mathbb E\Big[ \varphi''(X) \Big]^2 >0 \,.
    \end{align*}
    In addition, since $\varphi(x)$ satisfies Definition~\ref{def-denoiser-function}, Item~(1), we see that $\phi(u)$ is analytic for all $u \in (-0.9,0.9)$. This implies that
    \begin{align*}
        \lim_{k \to \infty} |c_k| \cdot \big( \tfrac{1}{2} \big)^k < \infty \,,
    \end{align*}
    which shows that $|c_k| \leq \Lambda\cdot 2^k$ for a constant $\Lambda$ and thus verifies Item~(1). Based on Item~(1), we immediately see that Item~(2) holds.
\end{proof}

We now describe the initialization. Let
\begin{equation}{\label{eq-def-varepsilon-0}}
    \varepsilon_0 = \phi(\tfrac{\rho}{2}) 
\end{equation}
and let $K_0 \in \mathbb N$ be a sufficiently large constant depending on $\rho$ such that 
\begin{equation}{\label{eq-def-K-0}}
    K_0 \geq 10^{30} \rho^{-30} |\phi''(0)|^4 \Lambda^4 \varepsilon_0^{-2} \mbox{ and } \frac{ \log(10^{-30}|\phi''(0)|^2\Lambda^2\rho^{20}K_0) }{ \log(10^{40}|\phi''(0)|^4\Lambda^{-4}\rho^{24}K_0\varepsilon_0^2) } < 1.01 \,.
\end{equation}
We then list all the sequences of length $K_0$ with distinct elements in $[n]$ as $\mathsf V_1,\ldots, \mathsf V_\mathtt M$ where $\mathtt M=\mathtt M(n,K_0)=n(n-1)\ldots (n-K_0+1)$. for each $\mathtt 1 \leq \mathtt i,\mathtt j \leq \mathtt M$, we will run a procedure of initialization and iteration for each $(\mathsf V_{\mathtt i},\mathsf V_{\mathtt j})$ and we know that for at least one of them (although we cannot decide which one it is a \emph{priori}) we are running an algorithm as if we have $K_0$ true pairs as seeds (i.e., $\mathsf V_{\mathtt j}=\pi(\mathsf V_{\mathtt i})$ and $\mathsf V_{\mathtt i} \cap (Q \cup S) = \mathsf V_{\mathtt j} \cap (R\cup T)=\emptyset$). For notation convenience, when describing the initialization and iteration we will drop $\mathtt i,\mathtt j$ from notations, but we should keep in mind that this procedure is applied to each pair $(\mathsf V_{\mathtt i},\mathsf V_{\mathtt j})$. With this clarified, we take a pair of fixed $\mathtt i,\mathtt j$ and denote $\mathsf V_{\mathtt i}=(u_1, \ldots, u_{K_0}), \mathsf V_{\mathtt j}=(v_1,\ldots,v_{K_0})$. Define two $(n-K_0)*K_0$ matrices $\widehat{f}^{(0)},\widehat{g}^{(0)}$ as  
\begin{equation}{\label{eq-def-initial-f,g-Gaussian}}
\begin{aligned}
    & \widehat{f}^{(0)}_{i,k} = \varphi\big( \widehat{\mathscr A}_{i,u_k} \big) \mbox{ for } i \in [n] \setminus \mathsf V_{\mathtt i}, k \in [K_0] \,; \\
    & \widehat{g}^{(0)}_{i,k} = \varphi\big( \widehat{\mathscr B}_{i,v_k} \big) \mbox{ for } i \in [n] \setminus \mathsf V_{\mathtt j}, k \in [K_0] \,.
\end{aligned}
\end{equation}

\subsection{Spectral subroutine}{\label{subsec:spec-subroutine}}

Now we further introduce a spectral subroutine which enables us to construct certain matrices independent of $(A',B')$ and the choice of $\mathsf V_{\mathtt i},\mathsf V_{\mathtt j}$, and we will show several spectral properties for the matrices we construct. Recall \eqref{eq-def-varepsilon-0}, \eqref{eq-def-K-0}. Define two $K_0*K_0$ matrices
\begin{equation}{\label{eq-def-Phi,Psi-0}}
    \Phi^{(0)} = \mathbb I \mbox{ and } \Psi^{(0)} = \varepsilon_0 \mathbb I \,.
\end{equation}
Now, assuming that we have already constructed $(\Phi^{(t)},\Psi^{(t)})$ {and $(\varepsilon_t, K_t)$} such that
\begin{equation}{\label{eq-spectral-assumption}}
    \begin{aligned}
        & \Phi^{(t)} \mbox{ has at least } \frac{3K_t}{4} \mbox{ eigenvalues between } 0.9 \mbox{ and } 1.1 \,; \\
        & \Psi^{(t)} \mbox{ has at least } \frac{3K_t}{4} \mbox{ eigenvalues between } 0.9\varepsilon_t \mbox{ and } 1.1\varepsilon_t 
    \end{aligned}    
\end{equation}
our algorithm will construct $(\Phi^{(t+1)},\Psi^{(t+1)})$ {and $(\varepsilon_{t+1}, K_{t+1})$} satisfying \eqref{eq-spectral-assumption} for $t+1$. Our first step of constructing $(\Phi^{(t+1)},\Psi^{(t+1)})$ is to show the following simple facts in linear algebra.
\begin{lemma}{\label{lem-existence-Xi-t}}
    Given any $K_t*K_t$ symmetric matrices $\Phi^{(t)},\Psi^{(t)}$ satisfying \eqref{eq-spectral-assumption}. There exists a matrix $\Xi^{(t)}$ of size $K_t*\frac{K_t}{12}$ such that the following conditions hold:
    \begin{enumerate}
        \item[(1)] $(\Xi^{(t)})^{\top} \Phi^{(t)} \Xi^{(t)}=\mathbb I_{K_t/12}$;
        \item[(2)] $(\Xi^{(t)})^{\top} \Psi^{(t)} \Xi^{(t)}$ is a diagonal matrix with diagonal entries in $(0.9\varepsilon_t,1.1\varepsilon_t)$.
    \end{enumerate}
\end{lemma}
The proof of Lemma~\ref{lem-existence-Xi-t} is incorporated in Section~\ref{subsec:existence-Xi} of the appendix. Now, based on Lemma~\ref{lem-existence-Xi-t}, we define
\begin{align}
    K_{t+1} &= 10^{-20} \rho^{20} |\phi''(0)|^2 \Lambda^{-2} K_t^2 \mbox{ for } t \geq 0 \,. \label{eq-def-K-t} \\
    \varepsilon_{t+1} &= \phi\Big( \tfrac{\rho}{2} \cdot \tfrac{12}{K_t} \mathrm{tr} \Big( \big(\Xi^{(t)}\big)^{\top} \Psi^{(t)} \Xi^{(t)} \Big) \Big)  \,. \label{eq-def-varepsilon-t}
\end{align}
Using Item~(2) in Lemma~\ref{lem-control-Taylor-expansion}, we see that when $\rho$ is sufficiently small we have 
\begin{equation}{\label{eq-bound-varepsilon_t}}
    \frac{\rho^2|\phi''(0)|}{2} \cdot \varepsilon_t^2 \geq \varepsilon_{t+1} \geq \frac{\rho^2|\phi''(0)|}{16} \cdot \varepsilon_t^2 \,.
\end{equation}
Thus, we have $K_{t+1}=\Theta(1) \cdot K_t^2$ and $\epsilon_{t+1}=\Theta(1) \cdot \epsilon_t^2$. In addition, let $\beta^{(t)}$ be a $\frac{K_t}{12}*K_{t+1}$ matrix whose entries are i.i.d.\ sampled uniformly from $\{ -\sqrt{12/K_t},\sqrt{12/K_t}\}$ (note that this sampling method ensures that the columns of $\beta^{(t)}$ are ``nearly orthogonal'' unit vectors). And define
\begin{equation}{\label{eq-def-Phi,Psi-t}} 
    \begin{aligned}
        &\Phi^{(t+1)}_{i,j} = \phi\Big(  \big(\beta^{(t)}_i\big)^{\top} \beta^{(t)}_j \Big) \,, \\
        &\Psi^{(t+1)}_{i,j} = \phi\Big( \tfrac{\rho}{2} \cdot \big(\beta^{(t)}_i \big)^{\top} \big(\Xi^{(t)}\big)^{\top} \Psi^{(t)} \Xi^{(t)} \beta^{(t)}_j \Big) \,.
    \end{aligned}
\end{equation}
The construction of $\Phi^{(t)}$ and $\Psi^{(t)}$ is to capture certain concentration phenomena in our iteration later. We refer the readers to Remark~\ref{remark-intuition-iteration} for the intuition behind such construction. We will show the following result which argues that $\Phi^{(t+1)},\Psi^{(t+1)}$ satisfies \eqref{eq-spectral-assumption} with positive probability.
\begin{lemma}{\label{lem-spectral-condition}}
    Let $K_t,\varepsilon_t$ be initialized as in \eqref{eq-def-K-0}, \eqref{eq-def-varepsilon-0} and inductively defined as in \eqref{eq-def-K-t}, \eqref{eq-def-varepsilon-t}. Suppose $\Phi^{(t)},\Psi^{(t)}$ satisfy \eqref{eq-spectral-assumption}. Then with probability at least $\frac{1}{2}$ over $\beta^{(t)}$ we have $\Phi^{(t+1)},\Psi^{(t+1)}$ satisfy \eqref{eq-spectral-assumption}.
\end{lemma}
The proof of Lemma~\ref{lem-spectral-condition} is incorporated in Section~\ref{subsec:spectral-preprocess} of the appendix. Based on Lemma~\ref{lem-spectral-condition}, since $K_t,\varepsilon_t$ and $\Phi^{(t)},\Psi^{(t)}$ are accessible by our algorithm, we can resample $\beta^{(t)}$ if the condition \eqref{eq-spectral-assumption} is not satisfied. This will increase the sampling complexity by a constant factor thanks to Lemma~\ref{lem-spectral-condition}. For this reason in what follows, we assume that we have performed resampling until \eqref{eq-spectral-assumption} is satisfied. Thus, throughout the rest part of the paper we will simply regard $\beta^{(t)}$ as fixed and our results hold for general $\beta^{(t)}$ as long as \eqref{eq-spectral-assumption} is satisfied.

\subsection{Iteration}{\label{subsec:vector-AMP}}

We remind here again that we will run the iteration procedure for all pairs $\mathsf V_{\mathtt i},\mathsf V_{\mathtt j}$. Recall \eqref{eq-def-initial-f,g-Gaussian}. Define iteratively $(n-K_0)*\frac{K_t}{12}$ matrices 
\begin{align*}
    &\widehat{h}^{(t)}_{a,b} = \sum_{ \substack{ c \in [n] \setminus \mathsf V_{\mathtt i}, 1 \leq d \leq K_t } } \tfrac{1}{\sqrt{n}} \widehat{\mathscr A}_{a,c} \widehat{f}^{(t)}_{c,d} {\Xi}^{(t)}_{d,b} \mbox{ for } a \in [n] \setminus \mathsf V_{\mathtt i},\ 1 \leq b \leq K_t/12 \,; \\
    &\widehat{\ell}^{(t)}_{a,b} = \sum_{ \substack{ c \in [n] \setminus \mathsf V_{\mathtt j}, 1 \leq d \leq K_t } } \tfrac{1}{\sqrt{n}} {\widehat{\mathscr B}}_{a,c} {\widehat{g}}^{(t)}_{c,d} {\Xi}^{(t)}_{d,b} \mbox{ for } a \in [n] \setminus \mathsf V_{\mathtt j},\ 1 \leq b \leq K_t/12 \,.
\end{align*}
Also iteratively define $(n-K_0)*K_{t+1}$ matrices
\begin{align*}
    &\widehat{f}^{(t+1)}_{a,b} = \varphi\Bigg( \sum_{ \substack{ 1 \leq c \leq K_t/12 } } \widehat{h}^{(t)}_{a,c} \beta^{(t)}_{c,b} \Bigg) \mbox{ for } a \in [n] \setminus \mathsf V_{\mathtt i},\ 1 \leq b \leq K_{t+1} \,; \\
    &\widehat{g}^{(t+1)}_{a,b} = \varphi\Bigg( \sum_{ \substack{ 1 \leq c \leq K_t/12 } } \widehat{\ell}^{(t)}_{a,c} \beta^{(t)}_{c,b} \Bigg) \mbox{ for } a \in [n] \setminus \mathsf V_{\mathtt j},\ 1 \leq b \leq K_{t+1} \,.
\end{align*}
In the matrix form, we can write
\begin{align}
    & \widehat{h}^{(t)} = \tfrac{1}{\sqrt{n}} \widehat{\mathscr A}_{([n]\setminus \mathsf V_{\mathtt i} \times [n]\setminus \mathsf V_{\mathtt i})} \widehat{f}^{(t)} \Xi^{(t)} \,, \quad \widehat{\ell}^{(t)} = \tfrac{1}{\sqrt{n}} \widehat{\mathscr B}_{ ([n]\setminus \mathsf V_{\mathtt j} \times [n]\setminus \mathsf V_{\mathtt j}) } \widehat{g}^{(t)} \Xi^{(t)} \,; \label{eq-def-iter-h-ell} \\
    & \widehat{f}^{(t+1)} = \varphi \circ \big( \widehat{h}^{(t)} \beta^{(t)} \big) \,, \quad \widehat{g}^{(t+1)} = \varphi \circ \big( \widehat{\ell}^{(t)} \beta^{(t)} \big) \,, \label{eq-def-iter-f,g}
\end{align}
where for a matrix $A=(A_{i,j})$ we use $\varphi \circ (A)$ to denote the matrix $(\varphi(A_{i,j}))$. 

\begin{remark}{\label{remark-intuition-iteration}}
    We now provide a brief explanation on the heuristics behind our iteration \eqref{eq-def-iter-h-ell}, \eqref{eq-def-iter-f,g}. Without loss of generality, we may assume that $\pi_*=\mathsf{id}$. The main intuition is that we expect the following concentration phenomenon. Informally speaking, we expect the following results hold:
    \begin{align}
        \tfrac{1}{n}\big( \widehat f^{(t)} \big)^{\top} \widehat f^{(t)}, \tfrac{1}{n} \big( \widehat g^{(t)} \big)^{\top} \widehat g^{(t)} \approx \Phi^{(t)}, \quad \tfrac{1}{n} \big( \widehat f^{(t)} \big)^{\top} \widehat g^{(t)} \approx \Psi^{(t)} \,. \label{eq-intuition}
    \end{align}
    To get a feeling about \eqref{eq-intuition}, let us assume that \eqref{eq-intuition} holds at time $t$ and try to verify \eqref{eq-intuition} for $t+1$ in a non-rigorous way. We first employ a non-rigorous simplification by regarding $\widehat f^{(t)}, \widehat g^{(t)}$ as fixed and simply ignore the adversarial corruption and the spectral preprocessing (i.e., by viewing $E,F=\mathbb O$ and $S,T=\emptyset$ in Section~\ref{subsec:preprocessing}). Under this simplification, by \eqref{eq-def-iter-h-ell} we see that $\widehat h^{(t)}$ and $\widehat \ell^{(t)}$ are two Gaussian matrices, with sample covariance structure given by
    \begin{align}
        & \mathbb E\Big[ \big( \widehat h^{(t)} \big)^{\top} \widehat h^{(t)} \Big] \overset{\eqref{eq-def-iter-h-ell}}{\approx} \tfrac{1}{n} \big( \Xi^{(t)} \big)^{\top} \big( \widehat f^{(t)} \big)^{\top} \widehat f^{(t)} \Xi^{(t)} \overset{\eqref{eq-intuition}}{\approx} \big( \Xi^{(t)} \big)^{\top} \Phi^{(t)} \Xi^{(t)} = \mathbb I_{K_t/12} \,; \label{eq-intuition-2} \\
        & \mathbb E\Big[ \big( \widehat \ell^{(t)} \big)^{\top} \widehat \ell^{(t)} \Big] \overset{\eqref{eq-def-iter-h-ell}}{\approx} \tfrac{1}{n} \big( \Xi^{(t)} \big)^{\top} \big( \widehat g^{(t)} \big)^{\top} \widehat g^{(t)} \Xi^{(t)} \overset{\eqref{eq-intuition}}{\approx} \big( \Xi^{(t)} \big)^{\top} \Phi^{(t)} \Xi^{(t)} = \mathbb I_{K_t/12} \,; \label{eq-intuition-3} \\
        & \mathbb E\Big[ \big( \widehat h^{(t)} \big)^{\top} \widehat \ell^{(t)} \Big] \overset{\eqref{eq-def-iter-h-ell}}{\approx} {\tfrac{\rho}{2} \cdot} \tfrac{1}{n} \big( \Xi^{(t)} \big)^{\top} \big( \widehat f^{(t)} \big)^{\top} \widehat g^{(t)} \Xi^{(t)} \overset{\eqref{eq-intuition}}{\approx} {\tfrac{\rho}{2} \cdot} ( \Xi^{(t)} )^{\top} \Psi^{(t)} \Xi^{(t)}  \,. \label{eq-intuition-4}
    \end{align}
    Thus, we further expect that
    \begin{align*}
        \tfrac{1}{n} \big( (\widehat f^{(t+1)})^{\top} \widehat f^{(t+1)} \big)_{i,j} &= \tfrac{1}{n} \sum_{u} \widehat f^{(t+1)}_{u,i} \widehat f^{(t+1)}_{u,j} \overset{\eqref{eq-def-iter-f,g}}{=} \tfrac{1}{n} \sum_{u} \varphi\Big( \sum_{k} \widehat h^{(t)}_{u,k} \beta^{(t)}_{k,i} \Big) \varphi\Big( \sum_{k} \widehat h^{(t)}_{u,k} \beta^{(t)}_{k,j} \Big) \\
        &\approx \mathbb E\Big[ \varphi(X)\varphi(Y) : X = \sum_{k} \widehat h^{(t)}_{u,k} \beta^{(t)}_{k,i}, Y = \sum_{k} \widehat h^{(t)}_{u,k} \beta^{(t)}_{k,j} \Big] \,,
    \end{align*}
    where in the ``$\approx$'' we use the law of large numbers. Note that $X,Y$ are approximately two normal random variables with variance and covariance given by 
    \begin{align*}
        \mathbb E[X^2] \approx (\beta^{(t)}_i)^{\top} \beta^{(t)}_i=1, \ \mathbb E[{Y}^2] \approx (\beta^{(t)}_j)^{\top} \beta^{(t)}_j=1, \ \mathbb E[XY] \approx (\beta^{(t)}_i)^{\top} \beta_j^{(t)} \,,
    \end{align*}
    where $\beta^{(t)} = (\beta^{(t)}_1,\ldots,\beta^{(t)}_{K_{t+1}})$. Similarly, we expect that
    \begin{align*}
        \tfrac{1}{n} \big( (\widehat f^{(t+1)})^{\top} \widehat g^{(t+1)} \big)_{i,j} &= \tfrac{1}{n} \sum_{u} \widehat f^{(t+1)}_{u,i} \widehat g^{(t+1)}_{u,j} \overset{\eqref{eq-def-iter-f,g}}{=} \tfrac{1}{n} \sum_{u} \varphi\Big( \sum_{k} \widehat h^{(t)}_{u,k} \beta^{(t)}_{k,i} \Big) \varphi\Big( \sum_{k} \widehat \ell^{(t)}_{u,k} \beta^{(t)}_{k,j} \Big) \\
        &\approx \mathbb E\Big[ \varphi(X)\varphi(Y) : X = \sum_{k} \widehat h^{(t)}_{u,k} \beta^{(t)}_{k,i}, Y = \sum_{k} \widehat \ell^{(t)}_{u,k} \beta^{(t)}_{k,j} \Big] \,,
    \end{align*}
    where $X,Y$ are approximately two normal random variables with variance and covariance given by 
    \begin{align*}
        \mathbb E[X^2] \approx (\beta^{(t)}_i)^{\top} \beta^{(t)}_i=1, \ \mathbb E[{Y}^2] \approx (\beta^{(t)}_j)^{\top} \beta^{(t)}_j=1, \ \mathbb E[XY] \approx {\tfrac{\rho}{2}} (\beta^{(t)}_i)^{\top} (\Xi^{(t)})^{\top} \Psi^{(t)} \Xi^{(t)} \beta_j^{(t)} \,,
    \end{align*}
    Thus, using \eqref{eq-def-Phi,Psi-t} we expect that \eqref{eq-intuition} holds for $t+1$.
\end{remark}

\begin{remark}{\label{remark-increase-dimension}}
    We now elaborate on the differences between our iteration and standard AMP. The first key difference is the introduction of multiplication by a time-dependent matrix $\Xi^{(t)}$, which actively increases the feature dimension. This increase is crucial for the algorithm's success. To see why, assume for simplicity that $\pi_*=\mathsf{id}$. Recall that our aim is that our constructed features $\widehat{h}^{(t)}, \widehat{\ell}^{(t)}$ capture the correlation in the sense that $\langle \widehat{h}^{(t)}_i, \widehat{\ell}^{(t)}_j \rangle$ is large if and only if $j=i$. Using \eqref{eq-intuition-2}--\eqref{eq-intuition-4}, at time $t$ we have 
    \begin{align*}
        & \big\langle \widehat{h}^{(t)}_i, \widehat{\ell}^{(t)}_i \big\rangle \mbox{ has variance } K_{t} \mbox{ and mean } K_{t} \varepsilon_{t}  \,;   \\
        & \big\langle \widehat{h}^{(t)}_i, \widehat{\ell}^{(t)}_j \big\rangle \mbox{ has variance } K_{t} \mbox{ and mean } 0 \mbox{ for } i \neq j \,.
    \end{align*}
    Thus, we have $\widehat{h}^{(t)}_i$ and $\widehat{\ell}^{(t)}_j$ are $K_t$-dimensional vectors with coordinate-wise correlation $\varepsilon_t$. However, as we shall see the per-coordinate correlation $\varepsilon_t$ inevitably decays during the iteration, regardless of the choice of denoiser. Consequently, to preserve a detectable signal in the inner product, we must compensate by increasing the feature dimension.
\end{remark}

\begin{remark}{\label{remark-no-Onsager-term}}
    Another key difference between our iteration and standard AMP iteration is the absence of the Onsager correction term, and we explain a bit more here. The standard AMP introduces the Onsager correction term to guarantee that the different iteration steps are approximately independent (i.e., the features $(\widehat{h}^{(t)}, \widehat{\ell}^{(t)})$ are approximately independent of $(\widehat{h}^{(t-1)}, \widehat{\ell}^{(t-1)})$). However, in our setting such approximate independence is already guaranteed due to our delicate choice of the matrix $\Xi^{(t)}$, i.e., we expect that
    \begin{equation}{\label{eq-intuition-independence}}
        \tfrac{1}{n} \big( \widehat f^{(t)} \big)^{\top} \widehat f^{(s)}, \ \tfrac{1}{n} \big( \widehat g^{(t)} \big)^{\top} \widehat g^{(s)}, \ \tfrac{1}{n} \big( \widehat f^{(t)} \big)^{\top} \widehat g^{(s)} \approx \mathbb O \mbox{ for } s<t \,. 
    \end{equation}
    To get a feeling about \eqref{eq-intuition-independence}, let us assume that \eqref{eq-intuition-independence} holds at time $t$ and try to verify \eqref{eq-intuition-independence} for $t+1$ in a non-rigorous way. Again, we will first employ a non-rigorous simplification by simply ignoring the adversarial corruption and spectral preprocessing (i.e., by viewing $E,F=\mathbb O$). Under this simplification, by \eqref{eq-def-iter-h-ell} we see that $\widehat h^{(t)},\widehat h^{(s)}$ and $\widehat \ell^{(t)},\widehat \ell^{(s)}$ are Gaussian matrices, with cross covariance structure given by
    \begin{align}
        & \mathbb E\Big[ \big( \widehat h^{(t)} \big)^{\top} \widehat h^{(s)} \Big] \overset{\eqref{eq-def-iter-h-ell}}{\approx} \tfrac{1}{n} \big( \Xi^{(t)} \big)^{\top} \big( \widehat f^{(t)} \big)^{\top} \widehat f^{(s)} \Xi^{(s)} \overset{\eqref{eq-intuition-independence}}{\approx} \mathbb O \,; \label{eq-intuition-independence-2} \\
        & \mathbb E\Big[ \big( \widehat \ell^{(t)} \big)^{\top} \widehat \ell^{(s)} \Big] \overset{\eqref{eq-def-iter-h-ell}}{\approx} \tfrac{1}{n} \big( \Xi^{(t)} \big)^{\top} \big( \widehat g^{(t)} \big)^{\top} \widehat g^{(s)} \Xi^{(s)} \overset{\eqref{eq-intuition-independence}}{\approx} \mathbb O \,; \label{eq-intuition-independence-3} \\
        & \mathbb E\Big[ \big( \widehat h^{(t)} \big)^{\top} \widehat \ell^{(s)} \Big] \overset{\eqref{eq-def-iter-h-ell}}{\approx} \tfrac{1}{n} \big( \Xi^{(t)} \big)^{\top} \big( \widehat f^{(t)} \big)^{\top} \widehat g^{(s)} \Xi^{(s)} \overset{\eqref{eq-intuition-independence}}{\approx} \mathbb O  \,. \label{eq-intuition-independence-4}
    \end{align}
    Thus, we further expect that
    \begin{align*}
        \tfrac{1}{n} \big( (\widehat f^{(t+1)})^{\top} \widehat f^{(s+1)} \big)_{i,j} &= \tfrac{1}{n} \sum_{u} \widehat f^{(t+1)}_{u,i} \widehat f^{(s+1)}_{u,j} \overset{\eqref{eq-def-iter-f,g}}{=} \tfrac{1}{n} \sum_{u} \varphi\Big( \sum_{k} \widehat h^{(t)}_{u,k} \beta^{(t)}_{k,i} \Big) \varphi\Big( \sum_{k} \widehat h^{(s)}_{u,k} \beta^{(s)}_{k,j} \Big) \\
        &\approx \mathbb E\Big[ \varphi(X)\varphi(Y) : X = \sum_{k} \widehat h^{(t)}_{u,k} \beta^{(t)}_{k,i}, Y = \sum_{k} \widehat h^{(s)}_{u,k} \beta^{(s)}_{k,j} \Big] \,,
    \end{align*}
    where in the ``$\approx$'' we use the law of large numbers. Note that $X,Y$ are approximately two normal random variables with variance and covariance given by 
    \begin{align*}
        \mathbb E[X^2]= (\beta^{(t)}_i)^{\top} \beta^{(t)}_i=1, \ \mathbb E[Y^2]= (\beta^{(t)}_j)^{\top} \beta^{(t)}_j=1, \ \mathbb E[XY] \approx 0 \,,
    \end{align*}
    where $\beta^{(t)} = (\beta^{(t)}_1,\ldots,\beta^{(t)}_{K_{t+1}})$. Thus, we expect that \eqref{eq-intuition-independence} holds for $t+1$. This eliminates the need for the Onsager correction, substantially simplifying the analysis. We note that similar techniques to avoid introducing Onsager correction terms have also been used in the field of orthogonal AMP \cite{ML17, CLM25+}.
\end{remark}


\subsection{Finishing and rounding}{\label{subsec:rounding}}

Define
\begin{equation}{\label{eq-def-t^*}}
    t^*= \min\Big\{ t \geq 0: K_t \geq (\log n)^{1.1} \Big\} \,.
\end{equation}
Using \eqref{eq-def-K-t} we see that
\begin{equation}{\label{eq-bound-K_t^*}}
    (\log n)^{1.1} \leq K_{t^*} \leq (\log n)^{2.2} \,.
\end{equation}
Recall that for each $1\leq \mathtt i,\mathtt j\leq \mathtt M$, we run the procedure of initialization and then run the AMP-iteration up to time $t^*$, and then we construct a permutation $\pi_{\mathtt i,\mathtt j}$ (with respect to $\mathsf V_{\mathtt i},\mathsf V_{\mathtt j}$) as follows. For $\mathsf V_{\mathtt i}=(u_1, \ldots, u_{K_0})$ and $\mathsf V_{\mathtt j}=(v_1,\ldots,v_{K_0})$ we set $\pi_{\mathtt i,\mathtt j}(u_k)=v_k$ for $1 \leq k \leq K_0$. And we let the restriction for $\pi_{\mathtt i,\mathtt j}$ on $[n] \setminus \mathtt V_{\mathtt i}$ to be the solution of 
\begin{align}
    \max\Big\{ \big\langle \widehat{h}^{(t^*)}, \widehat{\ell}^{(t^*)}(\sigma) \big\rangle \Big\} \mbox{ for all bijections } \sigma: [n] \setminus \mathsf V_{\mathtt i} \to [n] \setminus \mathsf V_{\mathtt j} \,. \label{eq-linear-assignment}
\end{align}
(If there are multiple maximizers we just arbitrarily choose one of them). Note that the above optimization problem \eqref{eq-linear-assignment} is a \emph{linear assignment problem}, which can be solved in time $O(n^3)$ by a linear program (LP) over doubly stochastic matrices or by the Hungarian algorithm \cite{Kuhn55}.

We say a pair of sequences $\mathsf V_{\mathtt i}=(u_1, \ldots, u_{K_0})$ and $\mathsf V_{\mathtt j}=(v_1,\ldots,v_{K_0})$ is \emph{a good pair} if
\begin{equation}{\label{eq-def-good-pair}}
    \mathsf V_{\mathtt i} \cap (Q\cup S) = \mathsf V_{\mathtt j} \cap (R\cup T) = \emptyset \mbox{ and } v_j = {\pi_*}(u_j) \mbox{ for } 1 \leq j \leq K_0 \,.
\end{equation}
The success of our algorithm lies in the following proposition which states that starting from a good pair we have that $\pi_{\mathtt i,\mathtt j}$ correctly recovers almost all vertices.

\begin{proposition}{\label{prop-almost-correct-matching}}
    Fix any sequence $\beta^{(t)}$ such that \eqref{eq-spectral-assumption} is satisfied. For any $\mathsf U,\mathsf V \subset [n]$ with cardinality $K_0$, define $\pi(\mathsf U,\mathsf V)=\pi_{\mathtt i,\mathtt j}$ if $(\mathsf U,\mathsf V)=(\mathsf V_{\mathtt i},\mathsf V_{\mathtt j})$. Then for a good pair $\mathsf U,\mathsf V$ we have with probability $1-o(1)$ over the randomness of $A$ and $B$,
    \begin{equation}{\label{eq-almost-correct-matching}}
        \#\{ v: \pi(\mathsf U,\mathsf V)(v) = \pi_*(v) \} \geq \big( 1-\tfrac{10}{\log n} \big) n \,.
    \end{equation}
\end{proposition}

Based on Proposition~\ref{prop-almost-correct-matching}, we will further employ a seeded graph matching algorithm to enhance an almost exact matching to an exact matching. Our matching algorithm is a simplified version of \cite[Algorithm 4]{BCL+19}. Let
\begin{align}
    & \alpha= \mathbb P(X \geq 1) \mbox{ where } X \overset{d}{=} \mathcal N(0,1) \,. \label{eq-def-alpha} \\
    &\psi(\rho) = \mathbb P( X \geq 1, Y \geq 1 ) \mbox{ where } (X,Y) \overset{d}{=} \mathcal N\Bigg( \begin{pmatrix} 0 & 0 \end{pmatrix}, \begin{pmatrix} 1 &\rho \\ \rho &1 \end{pmatrix} \Bigg) \,. \label{eq-def-psi-rho}
\end{align}
\begin{breakablealgorithm}{\label{algo:seeded-matching}}
    \caption{Seeded Matching Algorithm}
    \begin{algorithmic}[1]
    \STATE {\textbf{Input:}} A triple $(A',B',\Tilde{\pi},\rho)$ where $(A',B')$ are from the corrupted Gaussian Wigner model with correlation $\rho$, and $\Tilde{\pi}$ agrees with ${\pi_*}$ on $1-o(1)$ fraction of vertices.
    \STATE Define $\alpha$ as in \eqref{eq-def-alpha} and define $\psi(\rho)$ as in \eqref{eq-def-psi-rho}.
    \STATE Define $\Delta=\frac{\psi(\rho)n}{10}$ and set $\widehat{\pi} = \Tilde{\pi}$.
    \STATE For $u,v \in [n]$, define their $\widehat \pi$-neighborhood
    \begin{align*}
        N_{\widehat \pi}(u,v)= \sum_{w \in [n]} \Big( \mathbf 1_{ \{ A'_{u,w} \geq 1 \} } - \alpha \Big) \Big( \mathbf 1_{ \{ B'_{v,\widehat{\pi}(w)} \geq 1 \} } - \alpha \Big) \,.
    \end{align*}
    \STATE Repeat the following: if there exists a pair $u,v$ such that $N_{\widehat \pi}(u,v) \geq \Delta$ and $N_{\widehat \pi}(u,\widehat{\pi}(u))$, $N_{\widehat \pi}(\widehat{\pi}^{-1}(v),v) < \tfrac{\Delta}{10}$ (if there are several pairs satisfying this property, we simply choose arbitrary one of them), then modify $\widehat{\pi}$ to map $u$ to $v$ and map $\widehat{\pi}^{-1}(v)$ to $\widehat{\pi}(u)$; otherwise, move to Step 6.
    \STATE {\textbf{Output:}} $\widehat{\pi}$.
    \end{algorithmic}
\end{breakablealgorithm}

\begin{lemma}{\label{lem-boost-algorithm-works}}
    With probability $1-o(1)$ over the randomness of $A$ and $B$, for all possible $\widetilde{\pi} \in \mathfrak S_n$ that agrees with ${\pi_*}$ on at least $(1-\tfrac{10}{\log n})n$ coordinates, we have $\widehat{\pi}=\pi_*$.
\end{lemma}
The proof of Lemma~\ref{lem-boost-algorithm-works} is incorporated in Section~\ref{sec:rounding} of the appendix.
At this point, we can run Algorithm~\ref{algo:seeded-matching} for each $\pi_{\mathtt{i},\mathtt j}$ (which serves as input), and obtain the corresponding refined matching $\Hat{\pi}_{\mathtt{i},\mathtt j}$ (which is the output $\Hat{\pi}$). By Proposition~\ref{prop-almost-correct-matching}, we see that $\Hat{\pi}_{\mathtt{i},\mathtt j}= \pi_*$ with probability $1-o(1)$ if $(\mathsf V_{\mathtt i}, \mathsf V_{\mathtt j})$ is a good pair. Finally, we set 
\begin{align}{\label{equ-def-final-pi-hat}}
    \Hat{\pi}_\diamond = \arg \max_{ \Hat{\pi}_{\mathtt{i},\mathtt j} } \Bigg\{  \sum_{(u,v) \in E(V)}  \mathbf 1_{ \{ A'_{u,v} \geq 1 \} } \cdot \mathbf 1_{ \{ B'_{\Hat{\pi}_{\mathtt{i,j}}(u),\Hat{\pi}_{\mathtt{i,j}}(v)} \geq 1 \} }  \Bigg\}\,.
\end{align}
(If there are multiple maximizers we just arbitrarily choose one of them).
Our main result is the following theorem, which states that the estimator achieves exact matching with probability $1-o(1)$.
\begin{thm}{\label{main-thm}}
    Fix any sequence $\beta^{(t)}$ such that \eqref{eq-spectral-assumption} is satisfied. With probability $1-o(1)$ over the randomness of $A$ and $B$, we have $\Hat{\pi}_\diamond = \pi_*$.
\end{thm}

\subsection{Formal description of the algorithm and running time analysis} {\label{sec:formal-algorithm}}

In this section we present our algorithm formally. 

\begin{breakablealgorithm}{\label{algo:robust-matching}}
\caption{Robust Gaussian Matrix Matching Algorithm}
    \begin{algorithmic}[1]
    \STATE {\bf Input}: corrupted correlated Gaussian Wigner matrices $A',B'$, correlation $\rho$, corruption fraction $\epsilon$.
    \STATE Define $\widehat A', \widehat B'$ as in \eqref{eq-def-widehat-mathscr-A,B-Gaussian}. 
    \STATE Run Algorithm~\ref{alg:spectral-cleaning} with input $\widehat A',\widehat B'$ respectively; the output is denoted as $\widehat{\mathscr A}, \widehat{\mathscr B}$.
    \STATE Define $\phi, \mathtt{M}, K_0, \varepsilon_0,  \Phi^{(0)}, \Psi^{(0)}$ as above.
    \STATE Define $t^*$ as in \eqref{eq-def-t^*}.
    \STATE For $1 \leq t \leq t^*$ calculate $\Phi^{(t)},\Psi^{(t)},\Xi^{(t)}$ according to \eqref{eq-def-Phi,Psi-t}, \eqref{eq-def-Xi-t}; sample $\beta^{(t)}$ according to Lemma~\ref{lem-spectral-condition}.
    \STATE List all sequences with $K_0$ distinct elements in $[n]$ by $\mathsf{V}_1, \mathsf{V}_2, \ldots, \mathsf{V}_{\mathtt{M}}$.
    \FOR{$\mathtt{i,j}=1, \ldots, \mathtt{M}$}
    \STATE Define $\widehat f^{(0)},\widehat g^{(0)}$ as in \eqref{eq-def-initial-f,g-Gaussian}. 
    \STATE Set $\pi_{\mathtt{i},\mathtt j}(u_j) = v_j$ where $u_j, v_j$ are the $j$-th coordinate of $\mathsf V_{\mathtt i}, \mathsf{V}_{\mathtt{j}}$ respectively.
    \WHILE{ $K_t \leq {(\log n)^{1.1}}$ }
    \STATE Calculate $K_{t+1},\varepsilon_{t+1}$ according to \eqref{eq-def-K-t}, \eqref{eq-def-varepsilon-t}.
    \STATE Define $\widehat h^{(t)}, \widehat \ell^{(t)}, \widehat f^{(t+1)}, \widehat g^{(t+1)}$ for $1 \leq k \leq K_{t+1}$ according to \eqref{eq-def-iter-f,g}, \eqref{eq-def-iter-h-ell}; 
    \ENDWHILE
    \STATE Suppose we stop at $t=t^{*}$;
    \STATE Solve the linear assignment problem; the solution is denoted as $\pi_{\mathtt i,\mathtt j}$.
    \STATE Run Algorithm~\ref{algo:seeded-matching} with input $\pi_{\mathtt i,\mathtt j}$ and obtain $\widehat{\pi}_{\mathtt i,\mathtt j}$.
    \ENDFOR
    \STATE Find $\pi_{\mathtt{i,j}}^{*}$ which maximizes \eqref{equ-def-final-pi-hat}. 
    \STATE {\bf Output}: $\Hat{\pi}= {\pi}_{\mathtt{i,j}}^{*}$.
    \end{algorithmic}
\end{breakablealgorithm}

We now show that Algorithm~\ref{algo:robust-matching} runs in polynomial time.

\begin{proposition}{\label{prop-time-complexity}}
     The running time for computing each $\pi_{\mathtt i,\mathtt j}$ is $O(n^{3+o(1)})$. Furthermore, the running time for Algorithm~\ref{algo:robust-matching} is $O(n^{2K_0+3+o(1)})$.
\end{proposition}
\begin{proof}
    We first prove the first claim. Algorithm~\ref{alg:spectral-cleaning} takes time $O(n^{3+o(1)})$. We can compute $\widehat{f}^{(0)}, \widehat g^{(0)}$ in $O(K_0 n)$ time. Calculating $\Phi^{(t)},\Psi^{(t)},\Xi^{(t)}$ takes time 
    \begin{align*}
        \sum_{t \leq t^*}  O(K_{t}^3) = O(n^{o(1)}) \,.
    \end{align*}
    In addition, the iteration has $t^*=O(\log\log\log n)$ steps, and in each step for $t \leq t^*$ calculating $\widehat h^{(t)}, \widehat \ell^{(t)}, \widehat{f}^{(t+1)}, \widehat g^{(t+1)}$ takes $O(K_t n^2)$ time. Furthermore, in the linear assignment step calculating $\pi_{\mathtt{i,j}}$ takes $O(K_{t+1}^2 n^3)$ time and Algorithm~\ref{algo:seeded-matching} takes time $O(n^3)$. Therefore, the total amount of time spent on computing each $\pi_{\mathtt{i,j}}$ is upper-bounded by
    \begin{align*}
        O(K_0 n) + O(n^{o(1)}) + \sum_{t \leq t^*}  O(K_{t} n^2)  + O(K_{t^*}^2 n^3) + O(n^3) = O(n^{3+o(1)}) \,.
    \end{align*}
    We now prove the second claim. Since $\mathtt M \leq n^{K_0}$, the running time for computing all $\pi_{\mathtt{i,j}}$ is $O(n^{2K_0+3+o(1)})$. In addition, finding $\widehat{\pi}$ from $\{ \pi_{\mathtt{i,j}} \}$ takes $O(n^{2K_0+2})$ time. So the total running time is $O(n^{2K_0 +3+o(1)})$.
\end{proof}

It is straightforward to verify that Theorem~\ref{MAIN-THM} follows directly from Theorem~\ref{main-thm} and Proposition~\ref{prop-time-complexity}.

\section{Analysis of the algorithm}{\label{sec:analysis}}

\subsection{Heuristics}{\label{subsec:discussions}}

Before moving to the formal proof of Theorem~\ref{main-thm}, we outline some of the main ideas in analyzing our matching algorithm. Without loss of generality, we may assume throughout the rest of this paper that $\pi_*=\mathsf{id}$. Recall the heuristics behind our iteration explained in Remark~\ref{remark-intuition-iteration}. Using \eqref{eq-intuition-2}--\eqref{eq-intuition-4}, we see that at time $t^*$, we have 
\begin{align*}
    & \big\langle h^{(t^*)}_i, \ell^{(t^*)}_i \big\rangle \mbox{ has variance } K_{t^*} \mbox{ and mean } K_{t^*} \varepsilon_{t^*}  \,;   \\
    & \big\langle h^{(t^*)}_i, \ell^{(t^*)}_j \big\rangle \mbox{ has variance } K_{t^*} \mbox{ and mean } 0 \mbox{ for } i \neq j \,.
\end{align*}
Thus, the key quantity is the signal-to-noise ratio $\frac{(K_{t^*} \varepsilon_{t^*})^2}{K_{t^*}} = K_{t^*} \varepsilon_{t^*}^2$. Using \eqref{eq-def-K-t} and \eqref{eq-bound-varepsilon_t}, we see that
\begin{align}
    K_{t+1} \varepsilon_{t+1}^2 \geq & \Big( 10^{-20} \rho^{20} |\phi''(0)|^2 \Lambda^{-2} K_t^2 \Big) \cdot \Big( \frac{\rho^2 |\phi''(0)| }{ 16 } \varepsilon_t^2 \Big)^2 \nonumber \\
    = \ & \frac{ 10^{-20} \rho^{24} |\phi''(0)|^4 \Lambda^{-4} }{ 256 } \big(K_t \varepsilon_t^2 \big)^2 \,. \label{eq-iter-K-t-varepsilon-t-bound}
\end{align}
Using \eqref{eq-def-K-0} and \eqref{eq-def-varepsilon-0}, we see that $K_0 \varepsilon_0^2 \geq 10^{30} \Lambda^4 \rho^{-30} |\phi''(0)|^{-4}$ and thus $K_{t}\varepsilon_t^2$ is strictly increasing in $t$. In addition, from \eqref{eq-def-t^*} we have that 
\begin{align}
    K_{t^*} \varepsilon_{t^*}^2 &\geq \Big( \frac{ 10^{-20} \rho^{24} |\phi''(0)|^4 \Lambda^{-4} }{ 256 } K_0 \varepsilon_0^2 \Big)^{2^{t^*}} \overset{\eqref{eq-def-K-0}}{\geq} \Big( 10^{-20} \rho^{20} |\phi''(0)|^2 \Lambda^{-2} K_0 \Big)^{2^{t^*}/1.01} \nonumber \\
    & \overset{\eqref{eq-def-K-t}}{\geq} K_{t^*}^{1/1.01} \geq (\log n)^{1.01} \,, \label{eq-bound-signal-t^*}
\end{align}
which implies by a simple union bound that at $t^*$ the signal strength is strong enough to guarantee the correctness of $\widehat{\pi}$ on ``most'' of the coordinates. 

At this point, several obstacles prevent us from rigorously establishing the heuristic arguments in Remark~\ref{remark-intuition-iteration}. Most importantly, due to the adversarial corruptions $E,F$ and the subsequent spectral cleaning procedure, the matrices $\widehat{\mathscr A}$ and $\widehat{\mathscr B}$ are no longer Gaussian. To address this, our strategy is to first analyze the algorithm in a simplified ``clean'' setting where $E,F=\mathbb O$ and the spectral cleaning step is omitted. To be more precise, we fix a good pair $(\mathsf U,\mathsf V)$ and recall that $A'=A+E$ and $B'=B+F$. Also recall that we have sampled i.i.d.\ $\mathcal N(0,1)$ random variables $\{ G_{i,j}, H_{i,j}:1\leq i<j\leq n \}$. Define
\begin{equation}{\label{eq-def-mathscr-A,B-Gaussian}}
\begin{aligned}
    \mathscr A_{i,j} = \frac{ A_{i,j} + G_{i,j} }{ \sqrt{2} }, \mathscr B_{i,j} = \frac{ B_{i,j} + H_{i,j} }{ \sqrt{2} } \mbox{ for } i<j \,, \\
    \mathscr A_{i,j} = \frac{ A_{i,j} - G_{j,i} }{ \sqrt{2} }, \mathscr B_{i,j} = \frac{ B_{i,j} - H_{j,i} }{ \sqrt{2} } \mbox{ for } i>j \,.
\end{aligned}
\end{equation}
(Again, note that $\mathscr{A}$ and $\mathscr{B}$ are not symmetric matrices.) In addition, define $(f^{(0)},g^{(0)})=(\widehat f^{(0)}, \widehat g^{(0)})$ and let
\begin{align}
    & h^{(t)} = \tfrac{1}{\sqrt{n}} \mathscr A_{([n]\setminus \mathsf U \times [n]\setminus \mathsf U)} f^{(t)} \Xi^{(t)} \,, \quad \ell^{(t)} = \tfrac{1}{\sqrt{n}} \mathscr B_{ ([n]\setminus \mathsf V \times [n]\setminus \mathsf V) }  g^{(t)} \Xi^{(t)} \,; \label{eq-def-iter-h-ell-clean} \\
    & f^{(t+1)} = \varphi \circ \big( h^{(t)} \beta^{(t)} \big) \,, \quad g^{(t+1)} = \varphi \circ \big( \ell^{(t)} \beta^{(t)} \big) \,, \label{eq-def-iter-f,g-clean}
\end{align}
In this clean setting, $\mathscr{A}$ and $\mathscr{B}$ are indeed Gaussian matrices, allowing us to analyze the ``cleaned'' iteration $(f^{(t)},g^{(t)},h^{(t)},\ell^{(t)})$ using Gaussian projection techniques. These methods, common in the approximate message passing (AMP) literature (see, e.g., \cite{BM11}), systematically remove the influence of conditioning on previous steps by exploiting the fact that all conditioning events can be expressed as linear combinations of Gaussian variables. Generalizing this approach from a single clean matrix to two matrices with a sophisticated correlation structure is non-trivial; however, it is achievable, as demonstrated in \cite{DL22+}. Following this framework, we can show that the heuristics in Remark~\ref{remark-intuition-iteration} hold for the clean iteration $(f^{(t)}, g^{(t)}, h^{(t)}, \ell^{(t)})$. With a solid understanding of the cleaned iteration, we will then establish a suitable ``comparison'' theorem showing that $(\widehat f^{(t)},\widehat g^{(t)},\widehat h^{(t)},\widehat \ell^{(t)})$ and $(f^{(t)},g^{(t)},h^{(t)},\ell^{(t)})$ are ``close'' in a certain sense. This allows us to transfer results from the analyzable clean system to the more complicated corrupted system, thereby overcoming the previous obstacle.

\subsection{Proof of the main results}{\label{subsec:proof}}

The goal of this section is to prove Theorem~\ref{main-thm}. In addition, without loss of generality, we may assume that
\begin{equation}{\label{eq-assumption-epsilon}}
    \tfrac{1}{(\log n)^{100}} \leq \epsilon =o\Big( \tfrac{1}{(\log n)^{20}} \Big) \,.
\end{equation}
To this end, we first establish the following Lemma. 
\begin{lemma}{\label{lem-max-overlap}}
    With probability $1-o(1)$, for all $\sigma \in \mathfrak S_n$ we have 
    \begin{align*}
        \sum_{i,j=1}^{n} \mathbf 1_{ \{ A'_{i,j} \geq 1 \} } \cdot \mathbf 1_{ \{ B'_{\pi_*(i),\pi_*(j)} \geq 1 \} } \geq \sum_{i,j=1}^{n} \mathbf 1_{ \{ A'_{i,j} \geq 1 \} } \cdot \mathbf 1_{ \{ B'_{\sigma(i),\sigma(j)} \geq 1 \} }  \,,
    \end{align*}
    with equality holds if and only if $\sigma=\pi_*$.
\end{lemma}
The proof of Lemma~\ref{lem-max-overlap} is incorporated in Section~\ref{subsec:proof-lem-3.1} of the appendix. Provided with Lemma~\ref{lem-max-overlap}, we see that once we can show Proposition~\ref{prop-almost-correct-matching}, by the effectiveness of our seeded graph matching algorithm (see Lemma~\ref{lem-boost-algorithm-works}) we can deduce that we have $\widehat{\pi}_{\mathtt i,\mathtt j}=\pi_*$ for all good pairs $(\mathsf V_{\mathtt i},\mathsf V_{\mathtt j})$ and then we can deduce Theorem~\ref{main-thm} using Lemma~\ref{lem-max-overlap} and \eqref{equ-def-final-pi-hat}. 

The rest of this section is devoted to the proof of Proposition~\ref{prop-almost-correct-matching}. Recall \eqref{eq-def-mathscr-A,B-Gaussian}. As suggested by Section~\ref{subsec:discussions}, our approach is to first control the ``cleaned'' iteration $( f^{(t)}, g^{(t)}, h^{(t)}, \ell^{(t)} )$ in a delicate way and then establish proper ``comparison'' results to transfer our knowledge of $( f^{(t)}, g^{(t)}, h^{(t)}, \ell^{(t)} )$ to $( \widehat f^{(t)},\widehat g^{(t)},\widehat h^{(t)},\widehat \ell^{(t)} )$. To this end, we first show the following lemma. The first step of our proof is to establish a delicate control on $( f^{(t)}, g^{(t)}, h^{(t)}, \ell^{(t)} )$ for each $t$. To be more precise, write 
\begin{equation}\label{eq-def-Delta-s}
    \Delta_t = n^{-0.1} (\log n)^{10t} \prod_{i \leq t} K_i^{100}
\end{equation}
for $0 \leq t \leq t^*$. We will first show the following lemma:
\begin{lemma}{\label{lem-good-event-clean}}
    Denote $\mathcal E_t$ to be the following event:
    \begin{enumerate}
        \item[(1)] $\big\| \mathbb J_{(1\times [n] \setminus \mathsf U)} f^{(s)} \big\|_{\infty}, \big\| \mathbb J_{(1\times [n] \setminus \mathsf V)} f^{(s)} \big\|_{\infty} \leq \Delta_s n$ for $s \leq t$.
        \item[(2)] $\big\| \big(f^{(s)}\big)^{\top} f^{(s)} - \Phi^{(s)} \big\|_{\infty},\big\| \big(g^{(s)}\big)^{\top} g^{(s)} - \Phi^{(s)} \big\|_{\infty} \leq \Delta_s n$ for $s \leq t$.
        \item[(3)] $\big\| \big(f^{(s)}\big)^{\top} g^{(s)} - \Psi^{(s)} \big\|_{\infty} \leq \Delta_s n$ for $s \leq t$.
        \item[(4)] $\big\| \big(f^{(s)}\big)^{\top} g^{(r)} \big\|_{\infty},\big\| \big(f^{(r)}\big)^{\top} g^{(s)} \big\|_{\infty} \leq \Delta_{\max(s,r)} n$ for $s \neq r \leq t$.
        \item[(5)] $\big\| f^{(t)}_{ W \times [K_t] } \big\|_{\operatorname{HS}}, \big\| g^{(t)}_{ W \times [K_t] } \big\|_{\operatorname{HS}} \leq 100 \sqrt{K_t \epsilon \log(\epsilon^{-1}) n}$ for all $|W| \leq 10\epsilon n$.
        \item[(6)] $\big\| h^{(t)}_{ W \times [K_t] } \big\|_{\operatorname{HS}}, \big\| \ell^{(t)}_{ W \times [K_t] } \big\|_{\operatorname{HS}} \leq 100 \sqrt{K_t \epsilon \log(\epsilon^{-1}) n}$ for all $|W| \leq 10\epsilon n$.
        \item[(7)] $\#\big\{ i: \| h^{(t)}_i \| \geq \log\log n \big\}, \#\big\{ i: \| \ell^{(t)}_i \| \geq \log\log n \big\} \leq \frac{n}{\log n}$.
    \end{enumerate}
    We then have 
    \begin{equation}{\label{eq-E-t-holds}}
        \mathbb P( \cap_{0 \leq t \leq t^*} \mathcal E_t ) = 1-o(1) \,.
    \end{equation}
\end{lemma}
The proof of Lemma~\ref{lem-good-event-clean} is postponed to Section~\ref{subsec:proof-lem-3.2} of the appendix. In fact, it has been shown in \cite[Proposition~3.4]{DL22+} that Items~(1)--(4) hold for all $0 \leq t \leq t^*$ with probability $1-o(1)$ (although we need to make some slight modifications since we slightly simplified the iteration process). The main effort in this paper is to establish Items~(5)--(7). Based on Lemma~\ref{lem-good-event-clean}, we can show the following lemma which controls the behavior of $h^{(t^*)},\ell^{(t^*)}$.
\begin{lemma}{\label{lem-final-analysis-clean}}
    Denote $h^{(t^*)}=\big( h^{(t^*)}_1, \ldots, h^{(t^*)}_{n-K_0} \big)^{\top}$ and $\ell^{(t^*)}=\big( \ell^{(t^*)}_1, \ldots, \ell^{(t^*)}_{n-K_0} \big)^{\top}$. With probability $1-o(1)$ we have
    \begin{align*}
        & \langle h^{(t^*)}_i, \ell^{(t^*)}_i \rangle \geq \frac{9}{10} K_{t^*} \varepsilon_{t^*} \mbox{ for all } 1 \leq i \leq n-K_0 \,; \\
        & \big| \langle h^{(t^*)}_i, \ell^{(t^*)}_j \rangle \big| \leq \frac{1}{10} K_{t^*} \varepsilon_{t^*} \mbox{ for all } 1 \leq i \neq j \leq n-K_0 \,.
    \end{align*}
\end{lemma}
The proof of Lemma~\ref{lem-final-analysis-clean} is incorporated in Section~\ref{sec:proof-lem-3.3} of the appendix. Now we need to establish the following lemma which shows that $\widehat f^{(t)}, \widehat g^{(t)}, \widehat h^{(t)}, \widehat \ell^{(t)}$ is ``close'' to $f^{(t)}, g^{(t)}, h^{(t)}, \ell^{(t)}$ in certain sense. The proof of the following lemma is incorporated in Section~\ref{sec:proof-approx-lem} of the appendix.
\begin{lemma}{\label{lem-approx-f,g,h,ell}}
    Define
    \begin{equation}{\label{eq-def-aleph}}
        \aleph_t = \prod_{s \leq t} \Big( 1000 \log(\epsilon^{-1})^2 K_s \Big)
    \end{equation}
    With probability $1-o(1)$ we have the following results: for all $0 \leq t \leq t^*$
    \begin{align}
        & \big\| \widehat f^{(t)} - f^{(t)} \big\|_{\Fop} \,, \big\| \widehat g^{(t)} - g^{(t)} \big\|_{\Fop} \leq \aleph_t \cdot \sqrt{\epsilon n} \,, \label{eq-approx-f,g} \\ 
        & \big\| \widehat h^{(t)} - h^{(t)} \big\|_{\Fop} \,, \big\| \widehat \ell^{(t)} - \ell^{(t)} \big\|_{\Fop} \leq 1000 \aleph_t \sqrt{K_t\log(\epsilon^{-1})} \cdot \sqrt{\epsilon n} \,, \label{eq-approx-h,ell}
    \end{align}
\end{lemma}

\subsection{Proof of Proposition~\ref{prop-almost-correct-matching}}{\label{sec:proof-main-prop}}

In this section we prove Proposition~\ref{prop-almost-correct-matching} using Lemmas~\ref{lem-good-event-clean}, \ref{lem-final-analysis-clean} and \ref{lem-approx-f,g,h,ell}. Note that using Lemma~\ref{lem-approx-f,g,h,ell}, we have
\begin{align*}
    \big\| \widehat h^{(t^*)} - h^{(t^*)} \big\|_{\Fop}, \big\| \widehat \ell^{(t^*)} - \ell^{(t^*)} \big\|_{\Fop} \leq \aleph_{t^*} \sqrt{\epsilon n} \leq \frac{\varepsilon_{t^*}\sqrt{n}}{10000(\log n)^2}  \,,
\end{align*}
where in the last inequality we use the fact that $\epsilon=o\big( \tfrac{1}{(\log n)^{20}} \big), t^*=O(\log\log\log n)$ and
\begin{align*}
    \aleph_{t^*} \varepsilon_{t^*}^{-1} \overset{\eqref{eq-def-K-t},\eqref{eq-bound-signal-t^*}}{\leq} K_{t^*}^{2} \log(\epsilon^{-1})^{2t^*}  \overset{\eqref{eq-def-t^*}}{\leq} (\log n)^{5} \ll \epsilon^{-1/2} \,.
\end{align*}
Thus, using Chebyshev's inequality we have
\begin{align}
    \#\Big\{ i: \big\| \widehat h^{(t^*)}_i - h^{(t^*)}_i \big\| \leq \frac{\varepsilon_{t^*}}{100} \Big\}, \#\Big\{ i: \big\| \widehat \ell^{(t^*)}_i - \ell^{(t^*)}_i \big\| \leq \frac{K_{t^*}\varepsilon_{t^*}}{100} \Big\} \leq \frac{n}{\log n} \,. \label{eq-bound-card-unapprox-set}
\end{align}
Recall Lemmas~\ref{lem-final-analysis-clean}. We define $\mathtt U$ to be the collection of $u \in [n]$ such that
\begin{align*}
    \big\langle \widehat h^{(t^*)}_u, \widehat \ell^{(t^*)}_u \big\rangle < \frac{K_{t^*}\varepsilon_{t^*}}{2} \,,
\end{align*}
and we define $\mathtt{E}$ to be the collection of directed edges $(u,w)\in [n] \times [n]$ (with $u\neq w$) such that 
\begin{align*}
    \big\langle \widehat h^{(t^*)}_u, \widehat \ell^{(t^*)}_w \big\rangle > \frac{K_{t^*}\varepsilon_{t^*}}{8} \,.
\end{align*}
It is clear that $\mathtt U$ and $\mathtt E$ will potentially lead to mis-matching for our algorithm in the finishing stage. In addition, from \eqref{eq-bound-card-unapprox-set} and Item~(7) in Lemma~\ref{lem-good-event-clean} we have the following observations:
\begin{enumerate}
    \item[(I)] $|\mathtt U|\leq \frac{2n}{\log n}$;
    \item[(II)] All subset of $\mathtt E$ has cardinality at most $\frac{2n}{\log n}$ if each vertex is incident to at most one edge in this subset.
\end{enumerate}
To this end, Let $V_{\mathrm{fail}} = \{ v \in [n] : \Hat{\pi}(v) \neq v \}= \{ w_1,\ldots,w_m \}$. Note that if $\widehat{\pi}(u)=v$ and $\widehat{\pi}(v)=w$ for some $u \neq v$ (it is possible that $u=w$), at least one of the following four events 
\begin{align*}
    \big\{ v \in \mathtt U \big\}, \big\{ (u,v) \in \mathtt E \big\}, \big\{ (v,w) \in \mathtt E \big\}, \big\{ (u,w) \in \mathtt E \big\} 
\end{align*}
must occur, since otherwise by setting
\begin{align*}
    \widetilde{\pi}(u)=w, \widetilde{\pi}(v)=v \mbox{ and } \widetilde{\pi}(w) = \widehat{\pi}(w) \mbox{ otherwise}
\end{align*}
will make
\begin{align*}
    \big\langle \widehat h^{(t^*)}, \widehat \ell^{(t^*)}(\widehat{\pi}) \big\rangle < \big\langle \widehat h^{(t^*)}, \widehat \ell^{(t^*)}(\widetilde{\pi}) \big\rangle \,.
\end{align*}
We then construct a directed graph $\overrightarrow{H}$ on vertices $\{ w_1, w_2, \ldots, w_m  \} \cup \mathtt U$ as follows: for each $v \in \{ w_1, w_2, \ldots, w_m  \}$, if the finishing step matches $v$ to some $u$ with $u \neq v$, then we connect a directed edge from $v$ to $u$. Note our algorithm will not match a vertex twice, so all vertices have in-degree and out-degree both at most 1. Thus, the directed graph $\overrightarrow{H}$ is a collection of non-overlapping directed cycles $\mathcal C_1,\ldots, \mathcal C_r$. Recall that each $w_k \not \in \mathtt U$ is incident to at least one edge in $\overrightarrow{H}$, we then have
\begin{align*}
    |\mathcal C_1| + \ldots + |\mathcal C_r| \geq \frac{m-|\mathtt U|}{2} \,.
\end{align*}
Now, for each $\mathcal C_i$, using the above argument we can easily verify that there exists at least $\frac{|\mathcal C_i \setminus \mathtt U|}{10}$ non-overlapping edges in $\mathtt E$ with endpoints in $\mathcal C_i$. Thus, we can get a matching with cardinality at least 
\begin{align*}
    \frac{|\mathcal C_1|+\ldots+|\mathcal C_r|-|\mathtt U|}{10} \geq \frac{m-3|\mathtt U|}{20} \,.
\end{align*}
By Observation~(II), we see that
\begin{align*}
    \frac{m-3|\mathtt U|}{20} \leq \frac{2n}{\log n} \,.
\end{align*}
Combined with Observation~(I), we have $m \leq 100n/\log n$, completing the proof.


\section{Conclusions and open problems}

In this work, we give a polynomial time approximate message passing algorithm for matching two correlated Gaussian matrices under adversarial principal minor corruptions. This work is inspired by the work of \cite{DL22+} and \cite{IS24+}, and we address several specific issues that arise in the setting of robust random graph matching as we elaborate below.

\emph{A robust spectral subroutine.} The original spectral design in \cite{DL22+} involves solving certain linear equations with coefficients that depend on all prior AMP iterations up to $t-1$. This dependence makes their approach highly sensitive to adversarial perturbations. Our key algorithmic contribution is a modified spectral subroutine that operates \emph{independently of the AMP iteration} while still preserving sufficient signal (see Section~\ref{subsec:spec-subroutine} for details). This modification enables us to introduce a time-dependent matrix multiplication step within the iteration, which simultaneously enlarges the feature dimension and cancels the correlation during the iteration. 

\emph{Handling sophisticated correlation structures.} The analysis in \cite{IS24+} assumes the data matrix is a ``clean'' GOE matrix. In contrast, our data matrix is two GOE matrices with sophisticated \emph{correlation structures}. Thus, a main difficulty in our analysis is to deal with the correlation structure and the adversarial corruption \emph{simultaneously}. In addition, our AMP algorithm has $\omega(1)$ iterative steps and we need to show the output only changes $O(\tfrac{1}{\mathrm{poly}(\log n)})$ fraction under adversarial perturbations (see Lemma~\ref{lem-approx-f,g,h,ell} for details). We achieve this by establishing a sequence of concentration bounds in Sections~\ref{sec:analysis}, allowing us to iteratively control both correlation and corruption effects.

\emph{A seeded graph matching step.} Finally, due to the aforementioned complications we are only able to show that our AMP algorithm constructs an almost exact matching. To obtain an exact matching, we will employ the method of seeded graph matching (see Algorithm~\ref{algo:seeded-matching}). Although our seeded graph matching algorithm is a modified version of \cite[Algorithm~4]{BCL+19}, analyzing it requires careful treatment under adversarial corruptions.

Our work also highlights several important directions for future research, which we discuss below.

\underline{Optimal corruption scale.} In this paper, we propose an efficient Gaussian matrix matching algorithm that is robust under ${\tfrac{n}{(\log n)^{20}}*\tfrac{n}{(\log n)^{20}}}$ size of adversarial corruptions. However, an interesting open problem is whether it is possible to develop Gaussian matching algorithms for any $\epsilon n * \epsilon n$ adversarial perturbations where $\epsilon$ is a small constant.

\underline{Sparse graphs.} Although our algorithm can be extended to correlated \ER graphs with edge density $q \in (0,1)$ being \emph{a constant}, to deal with the adversarial perturbations, our current design and analysis of the algorithm crucially relies on the fact that the two matrices are \emph{dense} (i.e., each column and row of the adjacency matrix have $n^{1-o(1)}$ non-zero entries) and cannot extend to the case where the average density of a graph $q=n^{-c+o(1)}$ for some $c>0$. In such sparse regimes, exact matching recovery is not feasible, as an adversarial perturbation could corrupt all edges incident to a single vertex. Nonetheless, it remains an open question whether near-exact matching recovery is still achievable by efficient algorithms in this regime. Perhaps an even more challenging case is when the average degree of the graph is a constant (i.e., $nq=O(1)$). In this case, if no adversarial perturbation occurs, it was shown in \cite{GML24,GMS24,MWXY21+,MWXY23} that efficient partial matching algorithms exist given the correlation $\rho>\sqrt{\alpha}$, where $\alpha\approx 0.338$ is the Otter's constant. An intriguing question is whether partial matching is still achievable when $o(n)$ edges in both graphs are adversarially corrupted.

\underline{Other graph models.} Another important direction is to find robust graph matching algorithms for other important correlated random graph models, such as the random geometric graph model \cite{WWXY22+, GL24+, EGMM24}, the random inhomogeneous graph model \cite{DFW23+} and the stochastic block model \cite{RS21, CDGL24+, CR24+, CDGL25+, CDGL26}. We emphasize that it is also important to propose and study correlated graph models based on important real-world and scientific problems, albeit the models do not appear to be ``canonical'' from a mathematical point of view.

\section*{Acknowledgment}

Z.L. thanks Hang Du and Shuyang Gong for stimulating discussions. {We also warmly thank the anonymous reviewers for their careful reading and helpful comments which lead to a significant improvement on exposition.} This work is partially supported by National Key R$\&$D program of China (No. 2023YFA1010102) and NSFC Key Program (Project No. 12231002).

\begin{appendix}

\section{Supplementary proofs in Section~\ref{sec:alg-and-discussions}}{\label{sec:supp-proofs-sec-2}}

\subsection{Proof of Lemma~\ref{lem-bound-card-T,T'}}{\label{sec:statement-spectral-cleaning-alg}}

This section is devoted to the proof of Lemma~\ref{lem-bound-card-T,T'}. Although similar arguments have been established in \cite[Lemma~3.5]{IS24+}, we still choose to present the whole formal proof here for completeness. Let $\mathscr{M}^{(1)}, \ldots, \mathscr{M}^{(t)}$ be the matrix $\mathscr{M}$ after each iteration of the ``while'' loop. Denote $Q^{(t)} \subset Q$ to be the set of non-zeroed out indices at $t$ and let $E^{(t)}$ be the restriction of $E$ on $Q^{(t)}$. Note that the iteration will terminate once $Q^{(t)}=\emptyset$. We will show that with probability $1-o(1)$ we will have $\| \mathscr{M}^{(t)} \|_{\op} \leq 10\sqrt{n}$ under at most $4\epsilon n$ iterations via the following lemma.

\begin{lemma}{\label{lem-characterize-top-eigenvector}}
    Suppose the iteration does not terminate at $t$. Let $v,u$ be the left and right singular eigenvector of $\mathscr{M}^{(t)}$ corresponding to the leading eigenvalue, respectively. Then with probability $1-o(1)$ we have
    \begin{align*}
        \sum_{ i \in Q^{(t)} } \frac{v_i^2 + u_i^2}{2} \geq \frac{1}{2} \,.
    \end{align*}
\end{lemma}
\begin{proof}
    Recall that we have assumed $\| M \|_{\op} \leq 2\sqrt{n}$ (which follows from the standard GOE spectral bound). Since the iteration does not terminate at $t$, we have $| v^{\top} \mathscr{M}^{(t)} u | = \| \mathscr M^{(t)} \|_{\op} > 10\sqrt{n}$. Let $\widetilde{v}$ be the restriction of $v$ in $Q^{(t)}$ and $\widetilde{u}$ be the restriction of $u$ in $Q^{(t)}$. We then have
    \begin{align*}
        \| E^{(t)} \|_{\op} \cdot  \| \widetilde v \| \| \widetilde u \| \geq  \widetilde{v}^{\top} E^{(t)} \widetilde{u} = v^{\top} E^{(t)} u = v^{\top} (\mathscr M^{(t)}-M) u \geq \| \mathscr{M}^{(t)} \|_{\op} - \| M \|_{\op}  \,.
    \end{align*}
    In addition, we have $\| E^{(t)} \|_{\op} \leq \| M^{(t)} - M \|_{\op} \leq \| M^{(t)} \|_{\op} + \| M \|_{\op}$. Thus, 
    \begin{align*}
         \frac{\| \widetilde v \|^2+ \| \widetilde u \|^2}{2} \geq \| \widetilde v \| \| \widetilde u \| \geq \frac{ \| \mathscr{M}^{(t)} \|_{\op} - \| M \|_{\op} }{ \| E^{(t)} \|_{\op} } \geq \frac{ \| \mathscr{M}^{(t)} \|_{\op} - \| M \|_{\op} }{ \| \mathscr{M}^{(t)} \|_{\op} + \| M \|_{\op} } \geq \frac{1}{2} \,,
    \end{align*}
    as desired.
\end{proof}

To prove that our ``while'' loop terminates in $4\epsilon n$ steps with probability $1-o(1)$, define the stopping time $\tau=\min\big\{ t \geq 0: \| \mathscr{M}^{(t)} \|_{\op} \leq 10 \sqrt{n} \big\}$. Now for each $t \leq \tau$, let $I_t$ be the indicator of whether index removed between $\mathscr{M}^{(t)}$ and $\mathscr{M}^{(t+1)}$ was in $Q$. Then we have conditioned on $\tau>t$ and $I_1,\ldots,I_{t-1}$, each $I_t$ is stochastically dominated by a Bernoulli random variable with parameter $\frac{1}{2}$. Thus, we have
\begin{equation*}
    \mathbb P\big( \tau \geq 4\epsilon n \big) \leq \mathbb P\big( I_1 + \ldots + I_{4\epsilon n} \leq \epsilon n \big) = o(1) \,. 
\end{equation*}

\subsection{Proof of Lemma~\ref{lem-existence-Xi-t}}{\label{subsec:existence-Xi}}

Assume \eqref{eq-spectral-assumption} holds for $t$. We may write the spectral decomposition 
\begin{equation}{\label{eq-spectral-decomposition}}
    \Phi^{(t)} = \sum_{i=1}^{K_t} \lambda_i^{(t)} \nu_i^{(t)} \Big( \nu_i^{(t)} \Big)^{\top} \mbox{ and } \Psi^{(t)} = \sum_{i=1}^{K_t} \mu_i^{(t)} \zeta_i^{(t)} \Big( \zeta_i^{(t)} \Big)^{\top} \,,
\end{equation}
where for $1 \leq i \leq \frac{3K_t}{4}$ we have $\lambda_i^{(t)} \in (0.9,1.1)$ and $\mu^{(t)}_i \in (0.9\varepsilon_t, 1.1\varepsilon_t)$ (in particular, these eigenvalues are not in sorted order). As shown in \cite[Equations~(2.10),(2.11)]{DL22+}, we can choose
\begin{align*}
    \eta^{(t)}_1, \ldots, \eta^{(t)}_{K_t/12} \in \mathrm{span} \Big\{ \nu_1^{(t)}, \ldots, \nu^{(t)}_{3K_t/4} \Big\} \cap \mathrm{span} \Big\{ \zeta_1^{(t)}, \ldots, \zeta^{(t)}_{3K_t/4} \Big\} 
\end{align*}
such that
\begin{align}
    & \eta^{(t)}_i \Phi^{(t)} \eta^{(t)}_j = \eta^{(t)}_i \Psi^{(t)} \eta^{(t)}_j =0 \mbox{ for } i \neq j \,, \label{eq-orthogonal} \\
    & \eta^{(t)}_i \Phi^{(t)} \eta^{(t)}_i = 1, 1.1\varepsilon_t \geq \eta^{(t)}_i \Psi^{(t)} \eta^{(t)}_i \geq 0.9\varepsilon_t \mbox{ for } 1\leq i \leq K_t/12 \,. \label{eq-unit}
\end{align}
Set $\Xi^{(t)}$ to be a $K_t*\frac{K_t}{12}$ matrix such that
\begin{equation}{\label{eq-def-Xi-t}}
    \Xi^{(t)} = 
    \begin{pmatrix}
        \eta^{(t)}_1 &\ldots &\eta^{(t)}_{\frac{K_t}{12}}
    \end{pmatrix} \,.
\end{equation}
Note that using \eqref{eq-orthogonal} and \eqref{eq-unit}, we see that 
\begin{align*}
    \big(\Xi^{(t)}\big)^{\top} \Phi^{(t)} \Xi^{(t)} = \mathbb I_{K_t/12} \,,
\end{align*}
and $\big(\Xi^{(t)}\big)^{\top} \Psi^{(t)} \Xi^{(t)}$ is a $\frac{K_t}{12}*\frac{K_t}{12}$ diagonal matrix with diagonal entries lie in $(0.9\varepsilon_t, 1.1\varepsilon_t)$.

\subsection{Proof of Lemma~\ref{lem-spectral-condition}}{\label{subsec:spectral-preprocess}}

Recall that we sample $\beta^{(t)}$ to be a $\frac{K_t}{12} * K_{t+1}$ matrix such that $\beta^{(t)}_{i,j}$ are i.i.d.\ uniform random variables in $\{ -\sqrt{12/K_t},+\sqrt{12/K_t} \}$. Denote $\beta^{(t)}=\big( \beta^{(t)}_1,\ldots, \beta^{(t)}_{K_{t+1}} \big)$. Also recall \eqref{eq-def-Phi,Psi-t} and Lemma~\ref{lem-existence-Xi-t}. Thus we have 
\begin{align*}
    \frac{12}{K_t}\mathrm{tr} \Big( \big(\Xi^{(t)}\big)^{\top} \Psi^{(t)} \Xi^{(t)} \Big) \in (0.9\varepsilon_t, 1.1\varepsilon_t) \,.
\end{align*}
We now present the proof of Lemma~\ref{lem-spectral-condition}. Our proof is based on induction and thus from now on we assume that Lemma~\ref{lem-spectral-condition} holds up to time $t$. We first need the following auxiliary result.
\begin{lemma}{\label{lem-control-random-sampling}}
    Recall that we sample $\beta^{(t)}$ to be a $\frac{K_t}{12} * K_{t+1}$ matrix with entries uniformly in $\big\{-\sqrt{12/K_t},+\sqrt{12/K_t}\big\}$. Also denote $\beta^{(t)}=\big( \beta^{(t)}_1,\ldots, \beta^{(t)}_{K_{t+1}} \big)$. With probability at least $\frac{1}{2}$ we have the following conditions hold:
    \begin{align}
        & \Big\| (\beta^{(t)})^{\top} \beta^{(t)} \Big\|_{\infty} \leq \sqrt{\log K_t/K_t}, \label{eq-beta-condition-1} \\
        & \Big\| (\beta^{(t)})^{\top} (\Xi^{(t)})^{\top} \Psi^{(t)} \Xi^{(t)} \beta^{(t)} \Big\|_{\infty} \leq 2\varepsilon_t\sqrt{\log K_t/K_t} \,; \label{eq-beta-condition-2} \\
        & \sum_{1 \leq i,j \leq K_{t+1}} \big( (\beta^{(t)}_i)^{\top} \beta^{(t)}_j \big)^4 \leq 100 K_{t+1}^2/K_t^2, \label{eq-beta-condition-3} \\
        & \sum_{1 \leq i,j \leq K_{t+1}} \big( (\beta^{(t)}_i)^{\top} (\Xi^{(t)})^{\top} \Psi^{(t)} \Xi^{(t)} \beta^{(t)}_j \big)^4 \leq 100\varepsilon_t^2 K_{t+1}^2/K_{t}^2 \,. \label{eq-beta-condition-4}
    \end{align}
\end{lemma}
\begin{proof}
    The proof of Lemma~\ref{lem-control-random-sampling} was incorporated in \cite[Proposition~2.4]{DL22+}, and we omit further details here.
\end{proof}

We are now finally ready to provide the proof of Lemma~\ref{lem-spectral-condition}.
\begin{proof}[Proof of Lemma~\ref{lem-spectral-condition}]
We first consider $\Phi$. By \eqref{eq-def-Phi,Psi-t} and Lemma~\ref{lem-control-Taylor-expansion}, we can write $\Phi$ as 
\begin{align*}
    \Phi = \mathbb{I} + \sum_{k=2}^{\infty} c_k \Phi_k \mbox{ with } \Phi_k(i,j) = \big\langle \beta^{(t)}_i, \beta^{(t)}_j \big\rangle^{k} \,.
\end{align*}
By Lemma~\ref{lem-control-random-sampling}, we have (also recall $c_2=\frac{1}{2}\varphi''(0)$)
\begin{align}
    \Big\| c_2 \Phi_2 \Big\|_{\Fop}^2 &= \sum_{i,j} \Big( c_2 \Phi_2(i,j) \Big)^2 \leq \sum_{i \neq j} c_2^2 \Big( \frac{12}{K_t} \big\langle \beta^{(t)}_i, \beta^{(t)}_j \big\rangle \Big)^4 \nonumber \\
    &\overset{\eqref{eq-beta-condition-3}}{\leq} 10^6 |\varphi''(0)|^2 \cdot \frac{K_{t+1}^2}{K^2_t} \overset{\eqref{eq-def-K-t}}{\leq} \frac{1}{4} \cdot 10^{-6}K_{t+1} \,. \label{eq-Phi-2-HS-norm}
\end{align}
In addition, by Lemmas~\ref{lem-control-random-sampling} and \ref{lem-control-Taylor-expansion}, we have 
\begin{align*}
    \Big\| \sum_{k=3}^{\infty} c_k \Phi_k \Big\|_\infty \leq \sum_{k=3}^{\infty} 2^k \Big( \frac{ 24 \sqrt{ \log K_t}}{\sqrt{K_t}} \Big)^k \leq \frac{ {10^6} (\log K_t)^{1.5} }{ K_t^{1.5} }\,.
\end{align*}
Thus we have (using $K_t \geq K_0 \geq 10^{24}$)
\begin{align}
    \Big\| \sum_{k=3}^{\infty} c_k \Phi_k \Big\|_{\Fop}^2 &\leq K_{t+1}^2 \Big\| \sum_{k=3}^{\infty} c_k \Phi_k \Big\|_{\infty}^2 \leq \frac{ 10^{12} K_{t+1}^2 (\log K_t)^3 }{ K_t^3 } \nonumber \\
    &\overset{\eqref{eq-def-K-t}}{\leq} \frac{ \Lambda^2 10^{12} \Lambda^2 (\log K_t)^3 }{ K_t } \cdot K_{t+1} \leq \frac{1}{4} 10^{-6} K_{t+1} \,. \label{eq-Phi-geq-3-HS-norm}
\end{align}
Using $\| \mathrm{A+B} \|^2_{\Fop} \leq 2 (\| \mathrm{A} \|_{\Fop}^2+\| \mathrm{B} \|_{\Fop}^2)$ for all $\mathrm{A}$ and $\mathrm{B}$, we have
\begin{align*}
    \Big\| \sum_{k=2}^{\infty} c_k \Phi_k \Big\|^2_{\Fop} \leq 2\Big( \Big\| c_2 \Phi_2 \Big\|^2_{\Fop} + \Big\| \sum_{k=3}^{\infty} c_k \Phi_k \Big\|^2_{\Fop} \Big) \leq 10^{-6} K_{t+1}\,.
\end{align*}
Applying \cite[Lemma~2.12]{DL22+}, we then have that
\begin{equation}\label{eq-Phi-3-eigenvalue-bound}
    \#\Big\{ l: \Big| \varsigma_l \Big( \sum_{k=2}^{\infty} c_k \Phi_k \Big) \Big| \geq 0.01 \Big\}  \leq 0.01 K_{t+1}\,.
\end{equation}
Using standard facts in linear algebra (see, e.g., \cite[Lemmas~2.10]{DL22+}), we can write $\sum_{k=2}^{\infty} \Phi_k = C + D$, where $\| C \|_{\mathrm{op}} \leq 0.01$ and $\mathrm{rank}(D) \leq 0.01 K_{t+1}$. Noting that $\Phi = (\mathbb{I} + C) + D$, we get from standard linear algebra facts that (see \cite[Lemmas~2.11]{DL22+})  
\begin{align*}
    & \varsigma_{0.99 K_{t+1}} (\Phi) \geq \varsigma_{K_{t+1}} ( \mathbb{I} + C ) \geq 0.99\,,  \\
    & \varsigma_{0.01 K_{t+1}+1} (\Phi) \leq \varsigma_{1} ( \mathbb{I} + C ) \leq 1.01\,.
\end{align*}
This shows that $\Phi$ has at least $0.98K_{t+1}$ eigenvalues in $(0.99,1.01)$.

We deal with $\Psi$ in a similar way. By \eqref{eq-def-varepsilon-t}, \eqref{eq-def-Phi,Psi-t} and Lemma~\ref{lem-control-Taylor-expansion}, we can write $\Psi$ as 
\begin{align*}
    \Psi = \varepsilon_t \mathbb{I} + \sum_{k=2}^{\infty} c_k \Psi_k \mbox{ with } \Psi_k(i,j) = \Big( (\beta^{(t)}_i)^{\top} (\Xi^{(t)})^{\top} \Psi^{(t)} \Xi^{(t)} \beta^{(t)}_j \Big)^k \,.
\end{align*}
Again by \eqref{eq-beta-condition-4}, we have
\begin{align}
    \Big\| c_2 \Psi_2 \Big\|_{\Fop}^2 &= \sum_{i,j} \Big( c_2 \Psi_2(i,j) \Big)^2 \leq \frac{4^2 \cdot 10^{5} \rho^4 \varepsilon^4_t}{2^4} \frac{K_{t+1}^2}{K^2_t} \nonumber \\
    &\overset{\eqref{eq-def-varepsilon-t}}{\leq} \frac{10^{12} \varepsilon^2_{t+1}}{ {|\phi''(0)|}^2 } \frac{K_{t+1}^2}{K_t^2} \overset{\eqref{eq-def-K-t}}{\leq} \frac{1}{4} 10^{-6} \varepsilon^2_{t+1} K_{t+1} \,, \label{eq-Psi-2-HS-norm}
\end{align}
By Lemmas~\ref{lem-control-random-sampling} and \ref{lem-control-Taylor-expansion},
\begin{align*}
    \Big\| \sum_{k=3}^{\infty} c_k \Psi_k \Big\|_\infty \leq \sum_{k=3}^{\infty} 2^k \Big( \frac{\rho}{2} \frac{ 24 \Lambda \varepsilon_t \sqrt{ \log K_t}}{\sqrt{K_t}}  \Big)^k \leq \frac{ 10^6 \rho^3 \varepsilon_t^3 \Lambda (\log K_t)^{1.5} }{ K_t^{1.5} } \,.
\end{align*}
Thus we have 
\begin{align}
    &\Big\| \sum_{k=3}^{\infty} c_k \Psi_k \Big\|^2_{\Fop} \leq K_{t+1}^2 \Big\| \sum_{k=3}^{\infty} c_k \Psi_k \Big\|^2_{\infty} \leq \frac{ 10^{12} \rho^6 \varepsilon_t^6 \Lambda^2 (\log K_t)^3 K_{t+1}^2 }{ K_t^3 } \nonumber \\
    \overset{\eqref{eq-def-K-t}}{\leq} \ &\frac{ 10^{12} \rho^4 \varepsilon_t^4 \Lambda^2 \Lambda^2 (\log K_t)^3 K_{t+1}^2 }{ K_t^3 } \overset{ \eqref{eq-def-varepsilon-t},\eqref{eq-def-K-t} }{\leq} \frac{ \varepsilon_{t+1}^2 (\log K_t)^3 }{ K_t } K_{t+1} \leq \frac{1}{4} 10^{-6} \varepsilon_{t+1}^2 K_{t+1} \,.  \label{eq-Psi-geq-3-HS-norm}
\end{align}
Combined with \eqref{eq-Psi-2-HS-norm}, it yields that
\begin{align*}
    \Big\| \sum_{k=2}^{\infty} c_k \Psi_k \Big\|^2_{\Fop} \leq 2 \Big( \Big\| c_2 \Psi_2 \Big\|^2_{\Fop} + \Big\| \sum_{k=3}^{\infty} c_k \Psi_k \Big\|^2_{\Fop} \Big) \leq 10^{-6} K_{t+1} \varepsilon_{t+1}^2 \,.
\end{align*}
By \cite[Lemma~2.12]{DL22+} the matrix $\sum_{k=2}^{\infty} c_k \Psi_k$ has at most $0.01K_{t+1}$ eigenvalues with absolute values larger than $0.01\varepsilon_{t+1}$. By \cite[Lemma~2.10]{DL22+}, we can write $\sum_{k=2}^{\infty} c_k \Psi_k = C + D$, where $\| C \|_{\mathrm{op}}\leq 0.01 \varepsilon_{t+1}$ and $\mathrm{rank}(D) \leq 0.01 K_{t+1}$. By \cite[Lemma~2.11]{DL22+}, we know $\Psi = ( \varepsilon_t \mathbb{I} + C ) + D$ satisfies $\varsigma_{0.99K_{t+1}} (\Psi) \geq 0.98 \varepsilon_{t+1}$ and $\varsigma_{0.01K_{t+1}+1} (\Psi) \leq 1.02 \varepsilon_{t+1}$. This completes the proof of the lemma.
\end{proof}

\subsection{Proof of Lemma~\ref{lem-boost-algorithm-works}}{\label{sec:rounding}}

This section is devoted to the proof of Lemma~\ref{lem-boost-algorithm-works}. Clearly, it suffices to show the following result:
\begin{lemma}{\label{lem-large-neighborhood}}
    With probability $1-o(1)$, for all $\sigma\in\mathfrak S_n$ such that $\sigma$ agrees on $\pi_*$ on at least $(1-\tfrac{10}{\log n})n$ vertices, we have
    \begin{align*}
        N_{\sigma}(u,\pi_*(u)) \geq 2\Delta \mbox{ for all } u \in [n] \mbox{ and } N_{\sigma}(u,v) \leq \frac{\Delta}{20} \mbox{ for all } v \neq \pi_*(u) \,.
    \end{align*}
\end{lemma}
\begin{proof}
    Recall that $A'_{i,j}=A_{i,j}$ for all $(i,j) \not \in Q \times Q$. In addition, let $Q'=\{ i \in [n]: \pi_*(i) \neq \sigma(i) \}$, we have $|Q'|\leq \tfrac{10n}{\log n}$ and $|Q|\leq \epsilon n$. Thus, for all $u \in [n]$ we have
    \begin{align}
        &\mathbb P\big( N_{\sigma}(u,\pi_*(u)) \leq\ 2\Delta \big) \nonumber \\
        \leq\ & \mathbb P\Big( \sum_{w \in [n] \setminus Q\cup Q'} \big( \mathbf 1_{ \{ A'_{u,w} \geq 1 \} } - \alpha \big) \big( \mathbf 1_{ \{ B'_{\pi_*(u),\pi_*(w)} \geq 0 \} } - \alpha \big) \leq 2.1\Delta \Big) \nonumber \\
        =\ & \mathbb P\Big( \sum_{v \in [n] \setminus Q\cup Q'} \big( \mathbf 1_{ \{ A_{u,w} \geq 1 \} } - \alpha \big) \big( \mathbf 1_{ \{ B_{\pi_*(u),\pi_*(w)} \geq 0 \} } - \alpha \big) \leq 2.1\Delta \Big) \nonumber \\
        \leq\ & e^{-\rho^2 n/100} \,, \label{eq-Neighbor-bound-1}
    \end{align}
    where in the first inequality we use the fact that $|Q|,|Q'| \ll \Delta$ and in the second inequality we used Bernstein's inequality \cite[Theorem~1.4]{DP09}. Similarly, for all $u\neq v \in [n]$ we have
    \begin{align}
        &\mathbb P\big( N_{\sigma}(u,\pi_*(v)) \geq \tfrac{\Delta}{10} \big) \nonumber \\
        \leq\ & \mathbb P\Big( \sum_{w \in [n] \setminus Q\cup Q'} \big( \mathbf 1_{ \{ A'_{u,w} \geq 1 \} } - \alpha \big) \big( \mathbf 1_{ \{ B'_{\pi_*(v),\pi_*(w)} \geq 0 \} } - \alpha \big) \geq \tfrac{\Delta}{20} \Big) \nonumber \\
        =\ & \mathbb P\Big( \sum_{w \in [n] \setminus Q\cup Q'} \big( \mathbf 1_{ \{ A_{u,w} \geq 1 \} } - \alpha \big) \big( \mathbf 1_{ \{ B_{\pi_*(v),\pi_*(w)} \geq 0 \} } - \alpha \big) \geq \tfrac{\Delta}{20} \Big) \nonumber \\
        \leq\ & e^{-\rho^2 n/100} \,, \label{eq-Neighbor-bound-2}
    \end{align}
    where in the third inequality we again used Bernstein's inequality. Then the desired result follows from a simple union bound.
\end{proof}

We now present the proof of Lemma~\ref{lem-boost-algorithm-works}.
\begin{proof}[Proof of Lemma~\ref{lem-boost-algorithm-works}]
    Note that for all $\widetilde{\pi} \in \mathfrak S_n$ such that $\widehat{\pi}$ agrees with $\pi_*$ on at least $(1-\tfrac{10}{\log n})n$ coordinates, we have
    \begin{equation}{\label{eq-character-N(u,v)}}
        \begin{aligned}
            &N_{\widehat{\pi}}(u,\pi_*(u)) \geq 2 \Delta - \frac{n}{\log n} > \Delta \mbox{ for all } u \,, \\
            \mbox{ and } &N_{\widehat{\pi}}(u,v) \leq \frac{\Delta}{20} + \frac{n}{\log n} < \frac{\Delta}{10} \mbox{ for all } v \neq \pi_*(u) \,.
        \end{aligned}
    \end{equation}
    Thus, in each update in Step~5 of Algorithm~\ref{algo:seeded-matching} will correct a mistaken coordinate, and thus Step~5 will terminate at a permutation $\widehat{\pi} \in \mathfrak S_n$ such that $\widehat{\pi}(u)=\pi_*(u)$ for all $\widetilde{\pi}(u)=\pi_*(u)$. Note that if there exists $u \neq v \in [n]$ such that $\widehat{\pi}(u)=\pi_*(v) \neq \pi_*(u)$, then using \eqref{eq-character-N(u,v)} Step~5 should not stop and correct $u$ to $\pi_*(u)$, this yields $\widehat{\pi}=\pi_*$ with probability $1-o(1)$.
\end{proof}

\section{Supplementary proofs in Section~\ref{sec:analysis}}{\label{sec:supp-proofs-sec-3}}

\subsection{Proof of Lemma~\ref{lem-max-overlap}}{\label{subsec:proof-lem-3.1}}

In this section we present the proof of Lemma~\ref{lem-max-overlap}.
Without loss of generality, we may assume that $\pi_*=\mathsf{id}$, the identity permutation. Denote $\overline{A}_{i,j} = \mathbf 1_{A_{i,j} \geq 1}$ and $\overline{B}_{i,j}=\mathbf 1_{B_{i,j}\geq 1}$. Define $\overline{A}'_{i,j}$ and $\overline{B}'_{i,j}$ in a similar manner. Note that for all $\pi \in \mathfrak S_n \setminus \mathsf{id}$, we have $\pi$ admits a cycle decomposition $\pi=\sqcup_{O \in \mathcal O(\pi)} O$. We then have (denote $N(\pi)=\#\{ i\in [n]:\pi(i) \neq i \}$)
\begin{align*}
    \sum_{i,j} \overline{A}'_{i,j} \overline{B}'_{i,j} - \sum_{i,j} \overline{A}'_{i,j} \overline{B}'_{\pi(i),\pi(j)} &\geq \sum_{i,j} \overline{A}_{i,j} \overline{B}_{i,j} - \sum_{i,j} \overline{A}_{i,j} \overline{B}_{\pi(i),\pi(j)} - \epsilon n \cdot N(\pi) \\
    &= \sum_{O \in \mathcal O(\pi)} Z_O - \epsilon n \cdot N(\pi) \,,
\end{align*}
where 
\begin{align*}
    Z_O = \prod_{(i,j) \in O} \overline{A}_{i,j} \big( \overline{B}_{i,j} - \overline{B}_{\pi(i),\pi(j)} \big) \,.
\end{align*}
Note that marginally $(\overline{A}_{i,j},\overline{B}_{i,j})$ are two centered Bernoulli random variables with parameter $\alpha$ and correlation $\phi(\rho)$. Thus, using \cite[Lemma~8]{WXY22} we have $\{ Z_O : O \in \mathcal O(\pi) \}$ are independent and
\begin{align*}
    \mathbb E[ e^{-Z_O} ] = (1-\alpha\phi(\rho))^{|O|/2} \,.
\end{align*}
Thus, we have
\begin{align*}
    & \mathbb P \Big( \sum_{i,j} \overline{A}'_{i,j} \overline{B}'_{i,j} - \sum_{i,j} \overline{A}'_{i,j} \overline{B}'_{\pi(i),\pi(j)} \leq 0 \Big) 
    \leq \mathbb P\Big( \sum_{O \in \mathcal O(\pi)} Z_O \leq \epsilon n \cdot N(\pi) \Big) \\
    \leq\ & e^{\epsilon n N(\pi)} \mathbb E\Big[ e^{ -\sum_{O \in \mathcal O(\pi)} Z_O} \Big] \leq e^{\epsilon n N(\pi)} \prod_{O \in \mathcal O(\pi)} (1-\alpha\phi(\rho))^{|O|/2} \\
    \leq\ & e^{\epsilon n N(\pi)} (1-\alpha\phi(\rho))^{nN(\pi)/2} \,.
\end{align*}
Thus, by a union bound we have
\begin{align*}
    & \mathbb P\Big( \exists \pi \in \mathfrak S_n \setminus \{ \mathsf{id}\} \,, \sum_{i,j} \overline{A}'_{i,j} \overline{B}'_{i,j} \leq \sum_{i,j} \overline{A}'_{i,j} \overline{B}'_{\pi(i),\pi(j)}  \Big) \\
    \leq\ & \sum_{k=1}^{n} e^{\epsilon nk} (1-\alpha\phi(\rho))^{nk} \cdot \#\{ \pi: N(\pi)=k \} \\
    \leq\ & \sum_{k=1}^{n} \binom{n}{k} e^{\epsilon nk} (1-\alpha\phi(\rho))^{nk} = o(1) \,,
\end{align*}
where in the last inequality we use $\epsilon=o(\tfrac{1}{(\log n)^4})$. This leads to Lemma~\ref{lem-max-overlap}.

\subsection{Proof of Lemma~\ref{lem-good-event-clean}}{\label{subsec:proof-lem-3.2}}

Now we prove Lemma~\ref{lem-good-event-clean} by induction. We first show that Items~(1)--(5) holds for time $t=0$. Recall \eqref{eq-def-initial-f,g-Gaussian} and $(f^{(0)},g^{(0)})=(\widehat f^{(0)},\widehat g^{(0)})$. We then have (denote $\mathsf U=\{ u_1,\ldots,u_{K_0} \}$ and $\mathsf V=\{ v_1,\ldots,v_{K_0} \}$)
\begin{align*}
    \Big( \mathbb J_{1\times[n] \setminus \mathsf U} f^{(0)} \Big)_k = \sum_{i \in [n] \setminus \mathsf U} \varphi(\widehat{\mathscr A}_{i,u_k}) = \sum_{i \in [n] \setminus \mathsf U} \varphi(\mathscr A_{i,u_k}) \,,
\end{align*}
where in the last equality we use the fact that $\mathsf U \cap (Q \cup S)=\emptyset$ and thus $\widehat{\mathscr A}_{i,u_k}= \mathscr A_{i,u_k}$. Note that from Definition~\ref{def-denoiser-function}, we have
\begin{align*}
    \Big\{ \varphi(\mathscr A_{i,u_k}): i \in [n] \setminus \mathsf U \Big\}
\end{align*}
are i.i.d.\ bounded random variables with mean zero and variance $1$. Thus, using Bernstein's inequality \cite[Theorem~1.4]{DP09} we see that 
\begin{align}
    \mathbb P\Big( \big| \big( \mathbb J_{1\times[n] \setminus \mathsf U} f^{(0)} \big)_k \big| > \Delta_0 n \Big) \leq e^{-n^{0.5}} \,. \label{eq-concentration-bound}
\end{align} 
Thus, from a union bound on $k$ we see that $\big\| \mathbb J_{1\times[n] \setminus \mathsf U} f^{(0)} \big\|_{\infty} \leq \Delta_0 n$ holds with probability $1-O(e^{-n^{0.1}})$. Similarly, we can show that $\big\| \mathbb J_{1\times[n] \setminus \mathsf U} g^{(0)} \big\|_{\infty} \leq \Delta_0 n$ holds with probability $1-O(e^{-n^{0.1}})$ and thus Item~(1) holds for $t=0$ with probability $1-O(e^{-n^{0.1}})$. In addition, recall \eqref{eq-def-Phi,Psi-0} we see that 
\begin{align*}
    \Big( (f^{(0)})^{\top} f^{(0)}-\Phi^{(0)} \Big)_{i,j}, \Big( (g^{(0)})^{\top} g^{(0)}-\Phi^{(0)} \Big)_{i,j}, \Big( (f^{(0)})^{\top} g^{(0)}-\Psi^{(0)} \Big)_{i,j}
\end{align*}
can be written as sums of i.i.d.\ mean-zero bounded random variables. For instance,
\begin{align*}
    \Big( (f^{(0)})^{\top} g^{(0)}-\Psi^{(0)} \Big)_{i,i} = \sum_{i \in [n] \setminus \mathsf U} \Big( \varphi(\mathscr A_{i,u_k}) \varphi(\mathscr B_{i,u_k}) - \varepsilon_0 \Big)
\end{align*}
(recall that we have assumed ${\pi_*}=\mathsf{id}$ and $\mathsf V={\pi_*}(\mathsf U)=\mathsf U$). Thus we can obtain similar concentration bounds as in \eqref{eq-concentration-bound}. This yields that Items~(2)--(4) hold for $t=0$ with probability $1-O(e^{-n^{0.1}})$. Finally, using Bernstein's inequality again, for all $|W|\leq \epsilon n$ we have
\begin{align*}
    &\mathbb P\Big( \big\| f^{(0)}_{ W \times [K_0] } \big\|_{\Fop} > 10\sqrt{K_0\epsilon\log(\epsilon^{-1})n} \Big) \\
    =\ & \mathbb P\Bigg( \sum_{1 \leq k \leq K_0} \sum_{i \in W} \varphi(\mathscr A_{i,u_k})^2 > 100 K_0 \epsilon\log(\epsilon^{-1}) n \Bigg) \\
    \leq\ & \exp\Big( -90K_0 \epsilon\log(\epsilon^{-1}) n \Big) \,.
\end{align*}
Since the enumerations of $W$ is bounded by 
\begin{align*}
    \sum_{k \leq \epsilon n} \binom{n}{k} \leq \exp\big( 2\epsilon\log(\epsilon^{-1})n \big) \,,
\end{align*}
we conclude by a union bound that we have $\big\| f^{(0)}_{ W \times [K_0] } \big\|_{\Fop} \leq 10 \sqrt{K_0\epsilon\log(\epsilon^{-1})n}$ with probability $1-O(e^{-\epsilon n})$. We can similarly show that $\big\| g^{(0)}_{ W \times [K_0] } \big\|_{\Fop} \leq 10 \sqrt{K_0\epsilon\log(\epsilon^{-1})n}$ with probability $1-O(e^{-\epsilon n})$. In conclusion, we have shown that
\begin{equation}{\label{eq-initialization-prob}}
    \mathbb P\Big( \mbox{Items~(1)--(5) hold for } t=0 \Big) \geq 1-O(e^{-n^{0.1}}) \,.
\end{equation}
Now we assume that Items~(1)--(5) in Lemma~\ref{lem-good-event-clean} hold up to time $t$ and Items~(6)--(7) hold up to time $t-1$ (we denote this event as $\widetilde{E}_t$). Our goal is to bound the probability that Items~(6)--(7) hold for time $t$ and Items~(1)--(5) hold for time $t+1$. To this end, define
\begin{equation}{\label{eq-def-mathcal-F-t}}
    \mathcal F_t := \sigma\Big\{ f^{(s)}, g^{(s)}, h^{(r)}, \ell^{(r)}: s \leq t, r \leq t-1 \Big\} \,.
\end{equation}
We will use the following key observation constructed in \cite{DL22+}, which characterized the conditional distribution of $h^{(t)}$ and $\ell^{(t)}$ given $\mathcal F_t$.
\begin{claim}{\label{claim-conditional-distribution}}
    We have 
    \begin{align}
        \big( h^{(t)}, \ell^{(t)} \big)\big|_{\mathcal F_t} \overset{d}{=} \big( \mathscr G^{(t)} +   \delta^{(t)}, \mathscr H^{(t)}+ \kappa^{(t)} \big) \,, \label{eq-conditional-distribution}
    \end{align}
    where $\mathscr G^{(t)}_{u,i}, \mathscr H^{(t)}_{u,i}$ are independent mean-zero normal random variables with variances $1+O\big( K_t^{20} \Delta_t \big)$, and $\delta^{(t)}_{u,i},\kappa^{(t)}_{u,i}$ are Gaussian random variables with
    \begin{align*}
        & \mathbb E\big[ (\delta^{(t)}_{u,i})^2 \big] =\mathbb E\big[ (\kappa^{(t)}_{u,i})^2 \big] = O\big( K_t^{40} \Delta_t^2 \big) \,.
    \end{align*}
\end{claim}
The proof of Claim~\ref{claim-conditional-distribution} is established \cite{DL22+} in which they take
\begin{align*}
    \varphi(x)=\mathbf 1_{ \{ |x| \geq 10 \} }-\mathbb P(|\mathcal N(0,1)| \geq 10) \,;
\end{align*}
their proof can be easily adapted to the case of all symmetric, mean-zero and bounded $\varphi$ and thus we omit further details here for simplicity. In particular, by a simple union bound we have
\begin{align}{\label{eq-bound-small-Gaussian}}
    \mathbb P\Big( |\delta^{(t)}_{u,i}|, |\kappa^{(t)}_{u,i}| \leq K_t^{20} (\log n)^2 \Delta_t \Big) \geq 1 - e^{-(\log n)^2} \,,
\end{align}
which we will assume to happen throughout the remaining part of this section.

\subsubsection{Proofs of Items~(6) and (7)}

We first show that Item~(6) holds for $t$. Note that conditioned on $\mathcal F_t$, we have
\begin{align*}
    \big\| h^{(t)}_{ W \times [K_t] } \big\|_{\Fop} = \big\| \mathscr G^{(t)}_{ W \times [K_t] } + \delta^{(t)}_{ W \times [K_t] } \big\|_{\Fop} \leq \big\| \mathscr G^{(t)}_{ W \times [K_t] } \big\|_{\Fop} + \big\| \delta^{(t)} \big\|_{\Fop} \,.
\end{align*}
Using \eqref{eq-bound-small-Gaussian}, we see that we have
\begin{align*}
    \big\| \delta^{(t)} \big\|_{\Fop} \leq \sqrt{K_t n} \cdot \big\| \delta^{(t)} \big\|_{\infty} \leq \sqrt{K_t n} \cdot (\log n)^3 K_t^{20} \Delta_t \,.
\end{align*}
Using \eqref{eq-def-Delta-s}, we see that it suffices to show that
\begin{align}
    \big\| \mathscr G^{(t)}_{ W \times [K_t] } \big\|_{\Fop} \leq 90 \sqrt{K_t \epsilon \log(\epsilon^{-1}) n} \mbox{ for all } |W| = 10 \epsilon n \,.  \label{eq-item-(6)-relax-1}
\end{align}
We now verify \eqref{eq-item-(6)-relax-1} via a union bound on $W$. For each fixed $|W|\leq \epsilon n$, using Chernoff's inequality we have
\begin{align*}
    \mathbb P\Big( \big\| \mathscr G^{(t)}_{ W \times [K_t] } \big\|_{\Fop} > 90 \sqrt{K_t \epsilon \log(\epsilon^{-1}) n} \Big) \leq \exp( -100 K_t \epsilon \log(\epsilon^{-1}) n ) \,,
\end{align*}
thus leading to \eqref{eq-item-(6)-relax-1} since the enumeration of $W$ is bounded by 
\begin{align*}
    \sum_{k \leq 10\epsilon n} \binom{n}{k} \leq \exp( 20\epsilon \log(\epsilon^{-1}) n )  \,.
\end{align*}
We can similarly show that $\big\| \ell^{(t)}_{ W \times [K_t] } \big\|_{\Fop} \leq 10\sqrt{K_t \epsilon \log(\epsilon^{-1}) n}$ for all $|W| \leq \epsilon n$. Now we focus on Item~(7). Write 
\begin{align*}
    (h^{(t)})^{\top} = \big( (h^{(t)}_i)^{\top}:i \in [n] \setminus \mathsf U \big) \mbox{ and } (\ell^{(t)})^{\top} = \big( (\ell^{(t)}_i)^{\top}:i \in [n] \setminus \mathsf V \big) \,.
\end{align*}
Note that 
\begin{align*}
    \big\| h^{(t)}_i \big\| = \big\| \mathscr G^{(t)}_i + \delta^{(t)}_i \big\| \leq \big\| \mathscr G^{(t)}_i \big\| + K_t \Delta_t \,.
\end{align*}
Thus, we have
\begin{align}
    &\mathbb P \Big( \#\big\{ i: \big\| h^{(t)}_i \big\| > \log\log n \big\} > \tfrac{n}{\log n} \Big) \nonumber \\
    \leq\ & \mathbb P \Big( \#\big\{ i: \big\| \mathscr G^{(t)}_i \big\| > \log\log n/2 \big\} > \tfrac{n}{\log n} \Big) \nonumber \\
    \leq\ & \mathbb P\Big( \mathrm{Binom}(n,e^{-(\log\log n)^2/2}) > \tfrac{n}{\log n} \Big) \leq e^{-n/\log n} \,. \label{eq-item-7}
\end{align}
Similarly we can show that
\begin{align*}
    \mathbb P \Big( \#\big\{ i: \big\| \ell^{(t)}_i \big\| > \log\log n \big\} > \tfrac{n}{\log n} \Big) \leq e^{-n/\log n} \,.
\end{align*}
Thus we have
\begin{equation}{\label{eq-prob-item-(6)}}
    \mathbb P\Big( \mbox{Items~(6) and (7) holds for } t \mid \widetilde{\mathcal E}_t \Big) \geq 1-O(e^{\epsilon n}) \,.
\end{equation}

\subsubsection{Proof of Item~(1)}

In this subsection we show that Item~(1) holds for $t+1$. Recall \eqref{eq-def-iter-f,g-clean}. We have conditioned on $\mathcal F_t$
\begin{align*}
    f^{(t)}_{u,i} &= \varphi\Big( \big( h^{(t)} \beta^{(t)} \big)_{u,i} \Big) = \varphi\Big( \sum_{j} h^{(t)}_{u,j} \beta^{(t)}_{j,i} \Big) \overset{d}{=} \varphi\Big( \sum_{j} \mathscr G^{(t)}_{u,j} \beta^{(t)}_{j,i} + \sum_{j} \delta^{(t)}_{u,j} \beta^{(t)}_{j,i} \Big) \\
    &= \varphi\Big( \sum_{j} \mathscr G^{(t)}_{u,j} \beta^{(t)}_{j,i} \Big) + O(1) \cdot \Big| \sum_{j} \delta^{(t)}_{u,j} \beta^{(t)}_{j,i} \Big| \\
    &= \varphi\Big( \sum_{j} \mathscr G^{(t)}_{u,j} \beta^{(t)}_{j,i} \Big) + O(K_{t+1} K_t^{20} (\log n)^2 \Delta_t) \,,
\end{align*}
where in the last equality we use \eqref{eq-bound-small-Gaussian}. Thus, we have (recall \eqref{eq-def-Delta-s})
\begin{align*}
    \Big( \mathbb J_{1\times [n] \setminus \mathsf U} f^{(t)} \Big)_{i} = \sum_{u \in [n] \setminus \mathsf U} \varphi\Big( \sum_{j} \mathscr G^{(t)}_{u,j} \beta^{(t)}_{j,i} \Big) + o( \Delta_{t+1} n) \,.
\end{align*}
Note that 
\begin{align*}
    \Big\{ \sum_{j} \mathscr G^{(t)}_{u,j} \beta^{(t)}_{j,i} : u \in [n] \setminus \mathsf U \Big\}
\end{align*}
are independent Gaussian random variables with mean zero and variance $1+O(K_t^{20} \Delta_t)$, (recall that $\varphi$ is symmetric and bounded) using Chernoff's inequality we have
\begin{align*}
    \mathbb P\Bigg( \sum_{u \in [n] \setminus \mathsf U} \varphi\Big( \sum_{j} \mathscr G^{(t)}_{u,j} \beta^{(t)}_{j,i} \Big) \geq \tfrac{\Delta_{t+1}}{2} n \Bigg) \leq \exp(-n^{0.1}) \,.
\end{align*}
Thus by a union bound we have $\big\| \mathbb J_{1\times [n] \setminus \mathsf U} f^{(t)} \big\|_{\infty}$ holds with probability $1-o(e^{-(\log n)^2})$. Similarly result holds for $\big\| \mathbb J_{1\times [n] \setminus \mathsf V} g^{(t)} \big\|_{\infty}$. Thus, we get that
\begin{equation}{\label{eq-bound-prob-item-1}}
    \mathbb P\Big( \mbox{Item~(1) holds for } t+1 \mid \widetilde{\mathcal E}_t \Big) \geq 1-O(e^{-n^{0.1}}) \,.
\end{equation}

\subsubsection{Proofs of Items~(2)--(4)}

In this subsection we show that Items~(2)--(4) hold for $t+1$. Recall that we have shown
\begin{align*}
    f^{(t+1)}_{u,i} = \varphi\Big( \sum_{j} \mathscr G^{(t)}_{u,j} \beta^{(t)}_{j,i} \Big) + O(K_{t+1} K_t^{20} (\log n)^2 \Delta_t) \,.
\end{align*}
Thus, combining the fact that $\varphi(x)$ is bounded by $1$ we have
\begin{align*}
    &\Big( \big( f^{(t+1)} \big)^{\top} f^{(t+1)} \Big)_{i,j} = \sum_{u \in [n] \setminus \mathsf U} f^{(t+1)}_{u,i} f^{(t+1)}_{u,j} \\
    =\ & \sum_{u \in [n] \setminus \mathsf U} \varphi\Big( \sum_{k} \mathscr G^{(t)}_{u,k} \beta^{(t)}_{k,i} \Big) \varphi\Big( \sum_{j} \mathscr G^{(t)}_{u,k} \beta^{(t)}_{k,j} \Big) + O(K_{t+1} K_t^{20} (\log n)^2 \Delta_t n) \\
    =\ & \sum_{u \in [n] \setminus \mathsf U} \varphi\Big( \sum_{k} \mathscr G^{(t)}_{u,k} \beta^{(t)}_{k,i} \Big) \varphi\Big( \sum_{j} \mathscr G^{(t)}_{u,k} \beta^{(t)}_{k,j} \Big) + o(\Delta_{t+1} n) \,,
\end{align*}
where in the last equality we use \eqref{eq-def-Delta-s}.
Note that 
\begin{align*}
    \Big\{ \varphi\Big( \sum_{k} \mathscr G^{(t)}_{u,k} \beta^{(t)}_{k,i} \Big) \varphi\Big( \sum_{j} \mathscr G^{(t)}_{u,k} \beta^{(t)}_{k,j} \Big) : u \in [n] \setminus \mathsf U \Big\}
\end{align*}
are independent bounded random variables, with 
\begin{align*}
    & \mathbb E \Big[ \varphi\Big( \sum_{k} \mathscr G^{(t)}_{u,k} \beta^{(t)}_{k,i} \Big) \varphi\Big( \sum_{j} \mathscr G^{(t)}_{u,k} \beta^{(t)}_{k,j} \Big) \Big] \\
    =\ & \mathbb E\Big[ \varphi(X)\varphi(Y): X,Y \sim \mathcal N(0,1+O(K_t^{20} \Delta_t)), \mathrm{Cov}(X,Y)= (1+O(K_t^{20} \Delta_t)) \langle \beta^{(t)}_i, \beta^{(t)}_j \rangle \Big] \\
    =\ & \phi(\langle \beta^{(t)}_i, \beta^{(t)}_j \rangle) + O(K_t^{20} \Delta_t) = \Phi^{(t+1)}_{i,j} + O(K_t^{20} \Delta_t) \,.
\end{align*}
Thus, using Bernstein's inequality we see that
\begin{align*}
    & \mathbb P\Bigg( \Big| \big( \big( f^{(t+1)} \big)^{\top} f^{(t+1)} \big)_{i,j} - n \Phi^{(t+1)}_{i,j} \Big| > \Delta_{t+1}n \Bigg) \\
    \leq\ & \mathbb P\Bigg( \Big| \sum_{u \in [n] \setminus \mathsf U} \varphi\Big( \sum_{k} \mathscr G^{(t)}_{u,k} \beta^{(t)}_{k,i} \Big) \varphi\Big( \sum_{j} \mathscr G^{(t)}_{u,k} \beta^{(t)}_{k,j} \Big) - n \Phi^{(t+1)}_{i,j} \Big| > \Delta_{t+1}n/2 \Bigg) \leq e^{-n^{0.1}} \,.
\end{align*}
Thus, using a union bound we see that 
\begin{align*}
    & \mathbb P\Big( \big\| \big( f^{(t+1)} \big)^{\top} f^{(t+1)} - n \Phi^{(t+1)}_{i,j} \big\|_{\infty} \leq \Delta_{t+1}n \Big) \geq 1- n^2 e^{-n^{0.1}} \,.
\end{align*}
Similar results also holds for $\big( g^{(t+1)} \big)^{\top} g^{(t+1)}$. Thus we have
\begin{equation}{\label{eq-bound-prob-item-2}}
    \mathbb P\big( \mbox{Item~(2) holds for } t+1 \mid \widetilde{\mathcal E}_t \big) \geq 1- 2n^2 e^{-n^{0.1}} \,.
\end{equation}
Similarly, we have
\begin{align*}
    & \Big( \big( f^{(t+1)} \big)^{\top} g^{(t+1)} \Big)_{i,j} \\
    =\ & \sum_{u \in [n] \setminus \mathsf U} \varphi\Big( \sum_{k} \mathscr G^{(t)}_{u,k} \beta^{(t)}_{k,i} \Big) \varphi\Big( \sum_{j} \mathscr H^{(t)}_{u,k} \beta^{(t)}_{k,j} \Big) + O(K_{t+1} K_t^{20} \Delta_t n) \,,
\end{align*}
where
\begin{align*}
    \Big\{ \varphi\Big( \sum_{k} \mathscr G^{(t)}_{u,k} \beta^{(t)}_{k,i} \Big) \varphi\Big( \sum_{j} \mathscr H^{(t)}_{u,k} \beta^{(t)}_{k,j} \Big) : u \in [n] \setminus \mathsf U \Big\}
\end{align*}
are independent bounded random variables with 
\begin{align*}
    \mathbb E \Big[ \varphi\Big( \sum_{k} \mathscr G^{(t)}_{u,k} \beta^{(t)}_{k,i} \Big) \varphi\Big( \sum_{j} \mathscr H^{(t)}_{u,k} \beta^{(t)}_{k,j} \Big) \Big] = \Psi^{(t+1)}_{i,j} + O(K_t^{20} \Delta_t) \,.
\end{align*}
Thus we have
\begin{equation}{\label{eq-bound-prob-item-3}}
    \mathbb P\big( \mbox{Item~(3) holds for } t+1 \mid \widetilde{\mathcal E}_t \big) \geq 1- 2n^2 e^{-n^{0.1}} \,.
\end{equation}
Furthermore, we control the concentration of $\| (f^{(s)})^{\top} f^{(t+1)} \|_{\infty}$. Note that under $\mathcal{F}_{t}$, $f^{(s)}$ is fixed for $s\leq t$. So, 
\begin{align*}
    \big( (f^{(s)})^{\top} f^{(t+1)} \big)_{i,j} = \sum_{u \in [n] \setminus \mathsf U} f^{(s)}_{i,u} \varphi\Big( \sum_{k} \mathscr G^{(t)}_{u,k} \beta^{(t)}_{k,j} \Big) + O(K_t^{20}\Delta_t n) \,,
\end{align*}
which can be handled similarly to that for $\big\| \mathbb J_{1 \times [n] \setminus \mathsf U} f^{(t+1)} \big\|_{\infty}$. We omit further details since the modifications are minor. In conclusion, we have shown that
\begin{equation}{\label{eq-bound-prob-item-4}}
    \mathbb P\big( \mbox{Item~(4) holds for } t+1 \mid \widetilde{\mathcal E}_t \big) \geq 1- 3n^2 e^{-n^{0.1}} \,.
\end{equation}

\subsubsection{Proof of Item~(5)}

In this section we prove that Item~(5) holds for time $t+1$. Recall again that
\begin{align*}
    f^{(t+1)}_{u,i} = \varphi\Big( \sum_{j} \mathscr G^{(t)}_{u,j} \beta^{(t)}_{j,i} \Big) + O(K_{t+1} K_t^{20} (\log n)^2 \Delta_t) \,.
\end{align*}
Thus, for all $|W| \leq 10\epsilon n$ we have
\begin{align*}
    \big\| f^{(t+1)}_{W \times [K_t]} \big\|_{\operatorname{HS}}^2 &= \sum_{ u \in W } \sum_{i \leq K_{t+1}} \Big( \varphi\Big( \sum_{j} \mathscr G^{(t)}_{u,j} \beta^{(t)}_{j,i} \Big)^2 + O(K_{t+1} K_t^{20} (\log n)^2 \Delta_t) \Big) \\
    &\leq \sum_{ u \in W } \sum_{i \leq K_{t+1}}  \varphi\Big( \sum_{j} \mathscr G^{(t)}_{u,j} \beta^{(t)}_{j,i} \Big)^2 + O(K_{t+1}^2 K_t^{20} (\log n)^2 \Delta_t n)  \,.
\end{align*}
Thus, it suffices to show that
\begin{align}
    \sum_{ u \in W } \sum_{i \leq K_{t+1}} \varphi\Big( \sum_{j} \mathscr G^{(t)}_{u,j} \beta^{(t)}_{j,i} \Big)^2 \leq 90 K_{t+1}^2 \epsilon \log(\epsilon^{-1}) n \mbox{ for all } |W| \leq 10 \epsilon n \,. \label{eq-item-5-relax-1}
\end{align}
For each fixed $|W| \leq 10\epsilon n$, note that 
\begin{align*}
    \Big\{ \varphi\Big( \sum_{j} \mathscr G^{(t)}_{u,j} \beta^{(t)}_{j,i} \Big)^2 : u \in W \Big\}
\end{align*}
are bounded independent random variables with mean bound by $1$. Thus, using Bernstein's inequality again we get that
\begin{align*}
    & \mathbb P\Big( \sum_{ u \in W } \sum_{i \leq K_{t+1}} \varphi\Big( \sum_{j} \mathscr G^{(t)}_{u,j} \beta^{(t)}_{j,i} \Big)^2 > 90 K_{t+1}^2 \epsilon \log(\epsilon^{-1}) n \Big) \\
    \leq\ & K_{t+1} \mathbb P\Big( \sum_{ u \in W } \varphi\Big( \sum_{j} \mathscr G^{(t)}_{u,j} \beta^{(t)}_{j,i} \Big)^2 > 90 K_{t+1} \epsilon \log(\epsilon^{-1}) n \Big) \leq e^{-90 \epsilon \log(\epsilon^{-1}) n} \,.
\end{align*}
This yields \eqref{eq-item-5-relax-1} since the enumeration of $W$ is bounded by 
\begin{align*}
    \sum_{k \leq 10\epsilon n} \binom{n}{k} \leq \exp( 20\epsilon\log(\epsilon^{-1}) n ) \,.
\end{align*}
We can similarly show that $\big\| g^{(t)}_{ W \times [K_t] } \big\|_{\operatorname{HS}} \leq 10\sqrt{K_t \epsilon \log(\epsilon^{-1}) n}$ for all $|W| \leq 10\epsilon n$. Thus we have
\begin{equation}{\label{eq-prob-item-(5)}}
    \mathbb P\Big( \mbox{Item~(5) holds for } t+1 \mid \widetilde{\mathcal E}_t \Big) \geq 1-O(e^{-\epsilon n}) \,.
\end{equation}

\subsubsection{Conclusion}
By putting together \eqref{eq-bound-small-Gaussian}, \eqref{eq-bound-prob-item-1}, \eqref{eq-prob-item-(6)}, \eqref{eq-bound-prob-item-2}, \eqref{eq-bound-prob-item-3}, \eqref{eq-bound-prob-item-4} and \eqref{eq-prob-item-(5)}, we have proved 
\begin{equation*}
    \mathbb P\big( \widetilde{\mathcal E}_{t+1} \mid \widetilde{\mathcal E}_{t} \big) \geq 1-O(e^{-(\log n)^2}) \,.
\end{equation*}
In addition, since $t^*+1 =O(\log\log\log n)$, our quantitative bounds imply that all these hold simultaneously for $0 \leq t \leq t^*+1$ except with probability $O(e^{-0.5(\log n)^2})$. This concludes Lemma~\ref{lem-good-event-clean}.

\subsection{Proof of Lemma~\ref{lem-final-analysis-clean}}{\label{sec:proof-lem-3.3}}

Now we can present the proof of Lemma~\ref{lem-final-analysis-clean} formally. Based on Lemma~\ref{lem-good-event-clean}, it remains to show that under $\mathcal E_{\diamond}=\cap_{t \leq t^*} \mathcal E_{t}$, we have 
\begin{align*}
    \mathcal T =\ & \Bigg( \cap_{ 1 \leq i \leq n } \Big\{ \big\langle h^{(t^*)}_i, \ell^{(t^*)}_i \big\rangle \geq \frac{9}{10} K_{t^*} \varepsilon_{t^*} \Big\} \Bigg) \\
    &\bigcap \Bigg( \cap_{ 1 \leq i \neq j \leq n } \Big\{ \big\langle h^{(t^*)}_i, \ell^{(t^*)}_j \big\rangle \leq \frac{1}{10} K_{t^*} \varepsilon_{t^*} \Big\} \Bigg)
\end{align*}
occurs with probability $1-o(1)$. Recall \eqref{eq-conditional-distribution} and \eqref{eq-bound-small-Gaussian}. Thus, we have
\begin{align*}
    \big\langle h^{(t^*)}_i, \ell^{(t^*)}_i \big\rangle | \mathcal F_{t^*-1} \overset{d}{=} \langle \mathscr G^{(t^*)}_i, \mathscr H^{(t^*)}_i \rangle + O(n^{-0.01}) \,.
\end{align*}
Thus, we get that
\begin{align*}
    \mathbb P\Big( \big\langle h^{(t^*)}_i, \ell^{(t^*)}_i \big\rangle \leq \frac{9}{10} K_{t^*} \varepsilon_{t^*}; \mathcal E_{\diamond} \Big) \leq \exp\big( -K_{t^*} \varepsilon_{t^*}^2/100 \big) \leq n^{-4} \,,
\end{align*}
and similarly
\begin{align*}
    \mathbb P\Big( \big\langle h^{(t^*)}_i, \ell^{(t^*)}_j \big\rangle \geq \frac{1}{10} K_{t^*} \varepsilon_{t^*}; \mathcal E_{\diamond} \Big) \leq \exp\big( -K_{t^*} \varepsilon_{t^*}^2/100 \big) \leq n^{-4} \,.
\end{align*}
Combining these two estimates, we get from a simple union bound that
\begin{align*}
    \mathbb P\big( \mathcal T; \mathcal E_{\diamond} \big) \geq 1- \tfrac{1}{n} \,,
\end{align*}
which concludes the proof of Lemma~\ref{lem-final-analysis-clean}.

\subsection{Proof of Lemma~\ref{lem-approx-f,g,h,ell}}{\label{sec:proof-approx-lem}}

In this section we prove Lemma~\ref{lem-approx-f,g,h,ell} formally. Using Lemma~\ref{lem-good-event-clean}, we may work under the event $\cap_{t \leq t^*} \mathcal E_t$.
Our proof is based on induction on $t$. Recall that we have $\widehat f^{(0)}=f^{(0)}$ and $\widehat g^{(0)}=g^{(0)}$. Now suppose \eqref{eq-approx-f,g} holds for $t$. Recall from \eqref{eq-def-Xi-t} that the columns of $\Xi^{(t)}$ are unit vectors, we have
    \begin{align}
        \sqrt{n}\big\| \widehat h^{(t)} - h^{(t)} \big\|_{\Fop} & \overset{\eqref{eq-def-iter-h-ell},\eqref{eq-def-iter-h-ell-clean}}{=} \Big\| \big( \widehat{\mathscr A}_{([n]\setminus \mathsf U \times [n] \setminus \mathsf U)} \widehat f^{(t)} - \mathscr A_{([n] \setminus \mathsf U \times [n] \setminus \mathsf U)} f^{(t)} \big) \Xi^{(t)} \Big\|_{\Fop} \nonumber \\
        &\leq \Big\| \widehat{\mathscr A}_{([n]\setminus \mathsf U \times [n] \setminus \mathsf U)} \widehat f^{(t)} - \mathscr A_{([n]\setminus \mathsf U \times [n] \setminus \mathsf U)} f^{(t)}  \Big\|_{\Fop} \cdot \| \Xi^{(t)} \|_{\op} \nonumber \\
        &\leq \sqrt{K_t} \cdot \Big\| \widehat{\mathscr A}_{([n]\setminus \mathsf U \times [n] \setminus \mathsf U)} \widehat f^{(t)} - \mathscr A_{([n] \setminus \mathsf U \times [n] \setminus \mathsf U)} f^{(t)} \Big\|_{\Fop} \,. \label{eq-approx-h,ell-relax-1}
    \end{align}
    In addition, using triangle inequality we have \eqref{eq-approx-h,ell-relax-1} is bounded by $\sqrt{K_t}$ times
    \begin{align}
        & \Big\| \widehat{\mathscr A}_{([n] \setminus \mathsf U \times [n] \setminus \mathsf U)} \big( \widehat f^{(t)} - f^{(t)} \big) \Big\|_{\Fop} + \Big\| \big( \widehat{\mathscr A}_{([n] \setminus \mathsf U \times [n] \setminus \mathsf U)} - \mathscr A_{([n]\times [n] \setminus \mathsf U)} \big) f^{(t)} \Big\|_{\Fop} \nonumber \\ 
        \leq\ & \big\| \widehat{\mathscr A}_{([n] \setminus \mathsf U \times [n] \setminus \mathsf U)} \big\|_{\op} \big\| \widehat f^{(t)} - f^{(t)} \big\|_{\Fop} + \Big\| \big( \widehat{\mathscr A}_{([n] \setminus \mathsf U \times [n] \setminus \mathsf U)} - \mathscr A_{([n] \setminus \mathsf U \times [n] \setminus \mathsf U)} \big) f^{(t)} \Big\|_{\Fop} \nonumber \\
        \leq\ & 10 \aleph_t \cdot n\sqrt{\epsilon} + \Big\| \big( \widehat{\mathscr A}_{([n] \setminus \mathsf U \times [n] \setminus \mathsf U)} - \mathscr A_{([n]\times [n] \setminus \mathsf U)} \big) f^{(t)} \Big\|_{\Fop} \,, \label{eq-approx-h,ell-relax-2}
    \end{align}
    where in the last inequality we use $\| \widehat{\mathscr A}_{([n] \setminus \mathsf U \times [n] \setminus \mathsf U)} \|_{\op} \leq \| \widehat{\mathscr A} \|_{\op} \leq 10\sqrt{n}$ and the induction hypothesis.
    Recall \eqref{eq-de-f-widehat-A'}--\eqref{eq-de-f-widehat-F}. Also recall \eqref{eq-def-mathscr-A,B-Gaussian} and \eqref{eq-def-widehat-mathscr-A,B-Gaussian}, we have
    \begin{align*}
        &\big(\widehat{\mathscr A}_{([n]\times [n] \setminus \mathsf U)} - \mathscr A_{([n]\times [n] \setminus \mathsf U)}\big)_{i,j} \\
        &= \begin{cases}
            \tfrac{\widehat E_{i,j} + \mathscr A_{i,j}}{\sqrt{2}}  \,, & (i,j) \in (Q \setminus S) \times (Q \setminus S) \,; \\
            \tfrac{\mathscr A_{i,j}}{\sqrt{2}} \,, & i \in S \mbox{ or } j \in S, (i,j) \not\in  (Q \setminus S) \times (Q \setminus S) \,; \\
            0 \,, & \mbox{otherwise} \,.
        \end{cases}
    \end{align*}
    Thus, we have
    \begin{align}
        & \Big( \widehat{\mathscr A}_{([n]\setminus \mathsf U \times [n] \setminus \mathsf U)} - \mathscr A_{ (Q \cup S \setminus \mathsf U) \times (Q \cup S \setminus \mathsf U) } \Big) \widehat f^{(t)} \nonumber \\
        =\ & \widehat E_{ (Q \setminus S) \times (Q \setminus S) } f^{(t)}_{ (Q \setminus S) \times [K_t] } + \mathscr A_{ ([n] \setminus (\mathsf U \cap S)) \times S } f^{(t)}_{S\times [K_t]} + \mathscr A_{S \times [n] \setminus \mathsf U} f^{(t)} \,. \label{eq-approx-h,ell-relax-3}
    \end{align}
    Note that $\widehat E_{ (Q \setminus S) \times (Q \setminus S) } = \widehat{\mathscr A}_{(Q \setminus S) \times (Q \setminus S)} - \mathscr A_{(Q \setminus S) \times (Q \setminus S)}$, we then have
    \begin{align*}
        \big\| \widehat E_{ (Q \cap S) \times (Q \cap S) } \big\|_{\op} \leq \big\| \widehat{\mathscr A} \big\|_{\op} + \big\| \mathscr A \big\|_{\op} \leq 20 \sqrt{n} \,.
    \end{align*}
    Thus, we have
    \begin{align}
        &\big\| E_{ (Q \cap S) \times (Q \cap S) } f^{(t)}_{ (Q \cap S) \times [K_t] } \big\|_{\Fop} \nonumber \\
        \leq\ & \big\| E_{ (Q \cap S) \times (Q \cap S) } \big\|_{\op} \cdot \big\| f^{(t)}_{ (Q \cap S) \times [K_t] } \big\|_{\Fop} \nonumber \\
        \leq\ & 20 \sqrt{n} \cdot 10\sqrt{K_t \epsilon \log(\epsilon^{-1}) n} = 200n \sqrt{\epsilon \log(\epsilon^{-1}) K_t} \,, \label{eq-approx-h,ell-relax-4}
    \end{align}
    where in the second inequality we used Item~(5) in Lemma~\ref{lem-good-event-clean}. Similarly, we also have
    \begin{align}
        \big\| \mathscr A_{ ([n] \setminus (\mathsf U \cap S)) \times S } f^{(t)}_{S\times [K_t]} \big\|_{\Fop} &\leq \big\| \mathscr A_{ ([n] \setminus (\mathsf U \cap S)) \times S } \big\|_{\op} \big\| f^{(t)}_{S\times [K_t]} \big\|_{\Fop} \nonumber \\
        &\leq 2\sqrt{n} \big\| f^{(t)}_{S\times [K_t]} \big\|_{\Fop} \leq 20n \sqrt{\epsilon \log(\epsilon^{-1}) K_t} \,. \label{eq-approx-h,ell-relax-5} 
    \end{align}
    Finally, we have
    \begin{align}
        \big\| \mathscr A_{S \times [n] \setminus \mathsf U} f^{(t)} \big\|_{\Fop} \overset{\eqref{eq-def-iter-h-ell}}{=} \sqrt{n} \cdot \big\| h^{(t)}_{ S \times [n] \setminus \mathsf U } \big\|_{\Fop} \leq 10n\sqrt{K_t \epsilon \log(\epsilon^{-1})} \,. \label{eq-approx-h,ell-relax-6}
    \end{align}
    Plugging \eqref{eq-approx-h,ell-relax-4}, \eqref{eq-approx-h,ell-relax-5} and \eqref{eq-approx-h,ell-relax-6} into \eqref{eq-approx-h,ell-relax-3} we get that
    \begin{align*}
        \big\| \widehat{\mathscr A}_{([n]\setminus \mathsf U \times [n] \setminus \mathsf U)} - \mathscr A_{ (Q \cup S \setminus \mathsf U) \times (Q \cup S \setminus \mathsf U) } \widehat f^{(t)} \big\|_{\Fop} \leq 300n \sqrt{\epsilon \log(\epsilon^{-1}) K_t}
    \end{align*}
    Combined with \eqref{eq-approx-h,ell-relax-2}, we see that
    \begin{equation}{\label{eq-approx-h,ell-relax-7}}
        \big\| \widehat h^{(t)} - h^{(t)} \big\|_{\Fop} \leq 1000 \aleph_t \cdot \sqrt{K_t \epsilon \log(\epsilon^{-1}) n} \,.
    \end{equation}
    Similarly we can show $\big\| \widehat \ell^{(t)} - \ell^{(t)} \big\|_{\Fop} \leq 1000 \aleph_t \cdot \sqrt{K_t \epsilon \log(\epsilon^{-1}) n}$. Thus we have \eqref{eq-approx-h,ell} holds for $t$. Recall \eqref{eq-def-iter-h-ell} and \eqref{eq-def-iter-h-ell-clean}. Using the fact that $\varphi'$ is uniformly bounded by $1$ we have
    \begin{align}
        \big\| \widehat f^{(t+1)} - f^{(t+1)} \big\|_{\Fop}^2 &= \sum_{i=1}^{n} \sum_{j=1}^{K_{t+1}} \Big( \varphi\big( h^{(t)} \beta^{(t)} \big)_{i,j} - \varphi\big( \widehat h^{(t)} \beta^{(t)} \big)_{i,j} \Big)^2 \nonumber \\
        &\leq \sum_{i=1}^{n} \sum_{j=1}^{K_{t+1}} \Big( \big( h^{(t)} \beta^{(t)} \big)_{i,j} - \big( \widehat h^{(t)} \beta^{(t)} \big)_{i,j} \Big)^2 \nonumber \\
        &= \big\| \big( \widehat h^{(t)} - h^{(t)} \big) \beta^{(t)} \big\|_{\Fop}^2 \leq \big\| \big( \widehat h^{(t)} - h^{(t)} \big) \big\|_{\Fop}^2 \big\| \beta^{(t)} \big\|_{\Fop}^2  \nonumber \\
        &\leq K_{t+1} \cdot \Big( 1000 \aleph_t \cdot \sqrt{K_t \epsilon (\log(\epsilon^{-1})) n} \Big)^2 \overset{\eqref{eq-def-aleph}}{\leq} \aleph_{t+1}^2 \epsilon n \,.  \label{eq-approx-f,g-relax-1}
    \end{align}
    We can similarly show that 
    \begin{align*}
        \big\| \widehat \ell^{(t+1)} - \ell^{(t+1)} \big\|_{\Fop}^2 \leq \aleph_{t+1}^2 \epsilon n \,.
    \end{align*}
    Thus we have \eqref{eq-approx-f,g} holds for $t+1$. This completes our induction.

\end{appendix}

\end{document}